\documentclass{article}

\usepackage[preprint]{neurips_2025}

\usepackage[utf8]{inputenc} % allow utf-8 input
\usepackage[T1]{fontenc}    % use 8-bit T1 fonts
\usepackage{hyperref}       % hyperlinks
\usepackage{url}            % simple URL typesetting
\usepackage{booktabs}       % professional-quality tables
\usepackage{amsfonts}       % blackboard math symbols
\usepackage{nicefrac}       % compact symbols for 1/2, etc.
\usepackage{microtype}      % microtypography
\usepackage{xcolor}         % colors
\usepackage{siunitx}

\usepackage{amsmath, amssymb, amsfonts} % For math symbols
\usepackage{bm} % For bold math symbols if needed (\bm{v})
\usepackage{tikz}
\usetikzlibrary{calc, arrows.meta, shapes.geometric}
\usetikzlibrary{shapes, arrows, positioning, fit, patterns, shadings}
\usetikzlibrary{shapes, arrows, positioning, fit, shapes.geometric}
\usepackage{calc}

\definecolor{latentblue}{HTML}{2196f3} % A blue color
\definecolor{targetgreen}{HTML}{4caf50} % A green color
\definecolor{guidingarrow}{HTML}{f44336} % A red color for guidance

\usepackage{subcaption}     % For subfigures
\usepackage{graphicx}
\usepackage{float}
\usepackage{amsthm}
\usepackage{multirow}
\usepackage{array}

\DeclareMathOperator*{\E}{\mathbb{E}}
\newcommand{\R}{\mathbb{R}}

\newcommand{\U}{\mathcal{U}}

\newtheorem{assumption}{Assumption}
\newtheorem{theorem}{Theorem}
\newtheorem{lemma}{Lemma}

\newtheorem{remark}{Remark} 
\newtheorem{proposition}{Proposition}
\usepackage{marvosym}

\newenvironment{propone}[1]
  {\renewcommand\theproposition{#1}\proposition}
  {\endproposition}

\newcommand\blfootnote[1]{%
  \begingroup
  \renewcommand\thefootnote{}\footnote{#1}%
  \addtocounter{footnote}{-1}%
  \endgroup
}

\title{Aligning Latent Spaces with Flow Priors}

\author{
Yizhuo Li\textsuperscript{1,2}, Yuying Ge\textsuperscript{2,\Letter}, Yixiao Ge\textsuperscript{2}, Ying Shan\textsuperscript{2}, Ping Luo\textsuperscript{1,\Letter} \\
\textsuperscript{1}The University of Hong Kong, \textsuperscript{2}ARC Lab, Tencent PCG \\
}

\begin{document}

\maketitle
\blfootnote{This paper is partially supported by the National Key R\&D Program of China No.2022ZD0161000.}
\begin{center}
\vspace{-2em}
{\small \textbf{Project Page:} \url{https://liyizhuo.com/align/}}
\end{center}

\maketitle

\begin{abstract}
    This paper presents a novel framework for aligning learnable latent spaces to arbitrary target distributions by leveraging flow-based generative models as priors. 
    Our method first pretrains a flow model on the target features to capture the underlying distribution. 
    This fixed flow model subsequently regularizes the latent space via an alignment loss, which reformulates the flow matching objective to treat the latents as optimization targets.
    We formally prove that minimizing this alignment loss establishes a computationally tractable surrogate objective for maximizing a variational lower bound on the log-likelihood of latents under the target distribution.
    Notably, the proposed method eliminates computationally expensive likelihood evaluations and avoids ODE solving during optimization.
    As a proof of concept, we demonstrate in a controlled setting that the alignment loss landscape closely approximates the negative log-likelihood of the target distribution.
    We further validate the effectiveness of our approach through large-scale image generation experiments on ImageNet with diverse target distributions, accompanied by detailed discussions and ablation studies. 
    With both theoretical and empirical validation, our framework paves a new way for latent space alignment.
\end{abstract}
  
\section{Introduction}

% Importance of feature / latent space alignment
Latent models like autoencoders (AEs) are a cornerstone of modern machine learning~\cite{hinton2006reducing,baldi2012autoencoders,li2023comprehensive,chen2023auto,mienye2025deep}. 
These models typically map high-dimensional observations to a lower-dimensional latent space, aiming to capture salient features and dependencies~\cite{liou2014autoencoder,meng2017relational}. 
A highly desirable property of latent models is that the latent space should have structural properties, such as being close to a predefined target distribution~\cite{salah2011contractive,kingma2013auto,yao2025reconstruction,chen2024softvq}.
Such structure can incorporate domain-specific prior knowledge~\cite{Khemakhem2019VariationalAA,Raissi2019PhysicsinformedNN}, enhance the interpretability of the latent space\cite{Higgins2016betaVAELB,Chen2016InfoGANIR,Kim2018DisentanglingBF}, and facilitate latent space generation~\cite{rombach2022high,li2024autoregressive,leng2025repa,wen2025principal,yu2024representation}.
While significant progress has been made, ensuring that the learned latent representations possess such desired structure remains a crucial challenge.

% Existing methods for alignment
Traditional approaches to enforcing distributional conformity often involve minimizing divergences like the Kullback-Leibler (KL) divergence~\cite{kingma2013auto,rombach2022high}.
However, KL can be restrictive, particularly when the target prior is only implicitly defined (e.g., by samples).
% Alternative strategies, such as adversarial training~\cite{goodfellow2020generative}, introduce training instabilities.
In latent generative modeling, the latent space is usually regularized with known prior distributions, such as the Gaussian distribution for Variational Autoencoders (VAE)~\cite{kingma2013auto,Esser2020TamingTF}, and the categorical distribution for Vector Quantized VAE (VQ-VAE)~\cite{van2017neural}.
Recent works~\cite{qiu2025robust,li2024imagefolder,li2024xq,chen2025masked,yao2025reconstruction,chen2024softvq} have proposed to use pre-trained feature extractors as target distribution and directly optimize the latent distances, which are shown to be effective but computationally expensive and require per-sample features.
% Propose the question and answer, 1 paragraph

Recent advances in flow-based generative models~\cite{lipman2022flow,liu2022flow} offer a promising avenue to capture complex target distributions.
In this work, we address the question: \textit{Can we efficiently align a learnable latent space to an arbitrary target distribution using a pre-trained flow model as a prior?} 
We answer this question affirmatively by proposing a novel framework that leverages a pre-trained flow model to define a computationally tractable alignment loss, which effectively guides the latents towards the target distribution.

\begin{figure}[tbp]
    \centering
    \includegraphics[width=\textwidth]{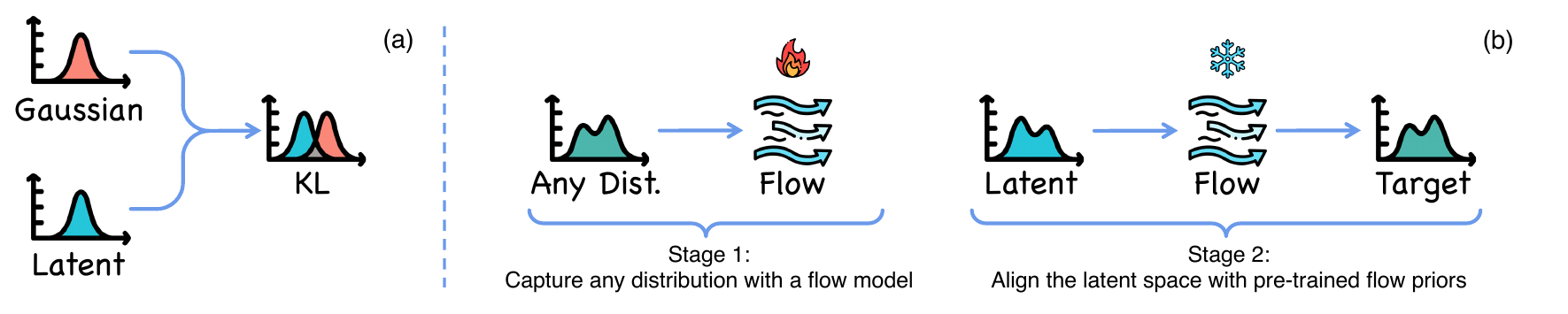}
    \caption{(a) Conventional alignment works with only known priors (e.g., Gaussian or categorical) using KL or cross-entropy losses. (b) Our proposed method can align the latent distribution to \textbf{arbitrary} target distribution captured by a pre-trained flow model.}
    \label{fig:teaser}
    % \vspace{-1em}
\end{figure}

% Method: Pipeline, 1 paragraph
Our proposed approach unfolds in a two-stage process as illustrated in Fig.\ref{fig:teaser}.
The first stage involves pre-training a flow-based generative model on the desired target features, allowing it to learn the mapping from a base distribution (e.g., Gaussian) to the target distribution. 
Once this flow model accurately captures the target distribution, its parameters are fixed. 
In the second stage, this flow model serves as a prior to regularize a learnable latent space, for instance, the output of the encoder in an AE. 
This regularization is achieved by minimizing an alignment loss, which ingeniously adapts the standard flow matching objective by treating the learnable latents as the target. 
This pipeline provides an efficient mechanism to guide the latent space towards the desired target structure without requiring direct comparison to target samples or expensive likelihood evaluations of the flow model.

% Method: Theoretical justification, 2 paragraphs
We theoretically justify our method by connecting the alignment loss to the maximum likelihood estimation of the latents under the target distribution. 
While directly maximizing this likelihood under a flow model is often computationally prohibitive due to the need to evaluate the trace of Jacobian determinants and solve an ordinary differential equation (ODE) for each step, our alignment loss offers a more tractable alternative. 
We formally demonstrate that minimizing this loss serves as a computationally efficient proxy for maximizing a variational lower bound on the log-likelihood of the latents under the flow-defined target distribution.

Our framework offers three key advantages.
First, our approach enables alignment to \textbf{arbitrary target distributions}, even those implicitly defined by samples, overcoming the limitations of conventional methods that require explicit parametric priors.
Second, the alignment loss acts as a \textbf{direct surrogate} for the log-likelihood of latents under the target distribution, providing a theoretically grounded objective that avoids heuristic metrics like cosine similarity used in per-sample feature matching~\cite{chen2025masked,yao2025reconstruction,chen2024softvq}.
Third, our framework is \textbf{computationally lightweight}, requiring only a single forward pass through the pre-trained flow model during training, thereby bypassing the need for expensive adversarial optimization~\cite{goodfellow2020generative}, likelihood evaluations, or per-sample feature extraction~\cite{qiu2025robust,li2024imagefolder,li2024xq}.

% Empirical validation, 1 paragraph
We empirically validate the efficacy of our proposed alignment strategy through a series of experiments. We start with illustrative experiments in a controlled toy setting using a mixture of Gaussians to confirm that our alignment loss landscape indeed serves as a proxy for the log-likelihood of the latents under the target distribution. Then we demonstrate the scalability of our approach by conducting large-scale image generation experiments on ImageNet~\cite{deng2009imagenet} with diverse target distributions. Detailed discussions and ablation studies are provided to underscore the robustness and effectiveness.

% Contribution, conclusion, future work, 1 paragraph
We believe this method offers a powerful and flexible tool for incorporating rich distributional priors into latent models. 
Our work paves the way for more flexible and powerful structured representation learning, and we anticipate its application and extension in various domains requiring distributional structure control over latent spaces.

\section{Related Work}

\subsection{Flow-based Models}
Flow-based generative models have emerged as a powerful class of generative models~\cite{esser2024scaling,flux2024,chen2025goku,Kingma2018GlowGF,Zhai2024NormalizingFA,Zhao2024FlowTurboTR,shin2025deeply}. They were first introduced as CNFs~\cite{chen2018neural,grathwohl2018ffjord} that learn an invertible mapping between a simple base distribution (e.g., Gaussian) and a complex data distribution. 
% By using a sequence of invertible transformations, these models allow for exact likelihood computation and efficient sampling.
Early works like NICE~\cite{Dinh2014NICENI} and Real NVP~\cite{Dinh2016DensityEU} introduced additive and affine coupling layers to construct invertible neural networks. A notable recent development is Flow Matching (FM)~\cite{lipman2022flow,albergo2022building,liu2022flow,neklyudov2023action,heitz2023iterative,tong2023improving}, which simplifies the training by regressing a vector field against a target field derived from pairs of samples, avoiding the need for simulating ODEs during training. In ICTM~\cite{Zhang2024FlowPF}, flow priors of generative models have been employed for MAP estimation to solve linear inverse problems. Our work leverages flow-based models to learn complex distributions as a prior for latent space alignment.

\subsection{Latent Space Alignment}
The alignment of latent spaces with predefined distributions is a crucial aspect of representation learning. In VAE~\cite{kingma2013auto}, the latent space is typically regularized to follow a standard Gaussian distribution. Several approaches have been proposed to use more flexible priors, such as hierarchical VAEs~\cite{Snderby2016LadderVA,Vahdat2020NVAEAD} or VAEs with inverse autoregressive flow (IAF) priors~\cite{Kingma2016ImprovedVI}. Another line of work focuses on aligning latent spaces with features extracted from pre-trained models~\cite{qiu2025robust,li2024imagefolder,li2024xq,chen2025masked,yao2025reconstruction,chen2024softvq,kim2025democratizing,zha2024language}. Our method differs by utilizing a pre-trained flow model to define an expressive target and a novel alignment loss, avoiding expensive likelihoods, adversarial training, or direct per-sample feature comparison.

\section{Preliminaries}

\subsection{Flow-Based Models}

We consider an ordinary differential equation (ODE) ideally defined by a time-dependent velocity field $\bm{u}(\bm{x}_t, t)$. The dynamics are given by:
\begin{equation}
    \label{eq:ode_ideal}
    \frac{\mathrm{d} \bm{x}_t}{\mathrm{d} t} = \bm{u}(\bm{x}_t, t), \quad \bm{x}_0 \sim p_{\mathrm{init}}, \quad \bm{x}_1 \sim p_{\mathrm{data}}
\end{equation}
Here, $p_{\mathrm{init}}$ is a simple prior distribution (e.g., a standard Gaussian distribution $\mathcal{N}(\bm{0}, \bm{I})$), and $p_{\mathrm{data}}$ is the target data distribution. We denote $\bm{x}_t \in \R^d$ as the state at time $t$, with $\bm{x}_0$ being the initial state and $\bm{x}_1$ the state at $t=1$. The velocity field $\bm{u}: \R^d \times [0,1] \to \R^d$ is assumed to be Lipschitz continuous in $\bm{x}$ and continuous in $t$ to ensure the existence and uniqueness of ODE solutions.

In practice, the ideal velocity field $\bm{u}$ is unknown. We approximate it with a parametric model, typically a neural network $\bm{v}_\theta(\bm{x}_t, t)$, parameterized by $\theta$. This defines a learned generative process:
\begin{equation}
    \label{eq:ode_model}
    \frac{\mathrm{d} \bm{x}_t}{\mathrm{d} t} = \bm{v}_\theta(\bm{x}_t, t), \quad \bm{x}_0 \sim p_{\mathrm{init}}
\end{equation}
For a given initial condition $\bm{x}_0$, the solution to this ODE, denoted by $\bm{x}_t = \bm{\Phi}_t^\theta(\bm{x}_0)$, is a trajectory (or flow) evolving from $\bm{x}_0$. The aim is to train $\bm{v}_\theta$ such that the $\bm{x}_1 = \bm{\Phi}_1^\theta(\bm{x}_0)$ matches $p_{\mathrm{data}}$.

Flow matching techniques~\cite{lipman2022flow,liu2022flow} train $\bm{v}_\theta$ by minimizing its difference from a target velocity field. This target field is often defined by constructing a probability path $p_t(\bm{x})$ that interpolates between $p_{\mathrm{init}}$ and $p_{\mathrm{data}}$. A common choice is a conditional path $\bm{x}_t(\bm{x}_0, \bm{x}_1)$ defined for pairs $(\bm{x}_0, \bm{x}_1)$ sampled from $p_{\mathrm{init}} \times p_{\mathrm{data}}$. For instance, Rectified Flow uses a linear interpolation: $\bm{x}_t(\bm{x}_0, \bm{x}_1) = (1-t) \bm{x}_0 + t \bm{x}_1$. The target velocity field corresponding to this path is $\bm{u}_t(\bm{x}_t | \bm{x}_0, \bm{x}_1) = \bm{x}_1 - \bm{x}_0$. The neural network $\bm{v}_\theta$ is then trained to predict this target field by minimizing the flow matching loss:
\begin{equation}
    \label{eq:loss}
    \mathcal{L}_{\text{FM}}(\theta) = \E_{t \sim \mathcal{U}[0, 1], \bm{x}_0 \sim p_{\mathrm{init}}, \bm{x}_1 \sim p_{\mathrm{data}}} \left[ \| \bm{v}_\theta((1-t)\bm{x}_0 + t \bm{x}_1, t) - (\bm{x}_1 - \bm{x}_0) \|^2 \right]
\end{equation}
In this paper, we consider $\bm{v}_\theta$ to be pre-trained, fixed, and optimal. That is, $\bm{v}_\theta$ is assumed to have perfectly minimized Eq.~\eqref{eq:loss}, such that $\bm{v}_\theta((1-t)\bm{x}_0 + t \bm{x}_1, t) = \bm{x}_1 - \bm{x}_0$ for all $\bm{x}_0 \sim p_{\mathrm{init}}$, $\bm{x}_1 \sim p_{\mathrm{data}}$, and $t \in [0,1]$. Such a $\bm{v}_\theta$ can serve as a regularizer to align latents to the target distribution.

\subsection{Likelihood Estimation with Flow Priors}

Let $p_1^{\bm{v}_\theta}(\bm{x}_1)$ denote the probability density at $t=1$ induced by the flow model $\bm{v}_\theta$ evolving from $p_{\mathrm{init}}$. Using the instantaneous change of variables formula, the log-likelihood of a sample $\bm{x}_1$ under this model can be computed by~\cite{chen2018neural,grathwohl2018ffjord}:
\begin{equation}
    \label{eq:likelihood}
    \log p_1^{\bm{v}_\theta}(\bm{x}_1) = \log p_{\mathrm{init}}(\bm{x}_0) - \int_0^1 \mathrm{Tr} (\nabla_{\bm{x}} \bm{v}_\theta(\bm{x}_s, s)) \mathrm{d} s
\end{equation}
Here, $\bm{x}_s = \bm{\Phi}_s^\theta(\bm{x}_0)$ is the trajectory generated by $\bm{v}_\theta$ starting from $\bm{x}_0$ and ending at $\bm{x}_1 = \bm{\Phi}_1^\theta(\bm{x}_0)$. Thus, $\bm{x}_0 = (\bm{\Phi}_1^\theta)^{-1}(\bm{x}_1)$ is obtained by flowing $\bm{x}_1$ backward in time to $t=0$.
Given a pre-trained flow model $\bm{v}_\theta$ that maps $p_{\mathrm{init}}$ (e.g., Gaussian noise) to a target distribution (e.g., target features), one can align new input samples with these target features by maximizing their log-likelihood under $p_1^{\bm{v}_\theta}$.
However, computing Eq.~\eqref{eq:likelihood} is often computationally expensive, primarily due to the trace of the Jacobian term ($\mathrm{Tr}(\nabla_{\bm{x}} \bm{v}_\theta)$) and the need for an ODE solver.
In this paper, we demonstrate that a similar alignment objective can be achieved by minimizing the flow matching loss Eq.~\eqref{eq:loss} with respect to $\bm{x}_1$, treating $\bm{x}_1$ as a variable to be optimized rather than a fixed sample from $p_{\mathrm{data}}$.

\section{Method}

In this paper, we aim to align a learnable latent space, whose latents are denoted by $\bm{y}$, to a target distribution $p_{\mathrm{data}}$. We first describe the overall pipeline in Sec.~\ref{sec:pipeline}. Our method leverages a pre-trained flow model to implicitly capture $p_{\mathrm{data}}$ and subsequently align the latents $\bm{y}$. Then, we provide an intuitive explanation in Sec.~\ref{sec:intuitive} and a formal proof of the proposed method in Sec.~\ref{sec:proof}.

\subsection{Pipeline}
\label{sec:pipeline}

Let $\bm{y} \in \R^{d_1}$ denote a sample from a learnable latent space. These latents $\bm{y}$ are typically produced by a parametric model $G_\phi$ (e.g., the encoder of an AE), whose parameters $\phi$ we aim to train. Let $\bm{x} \in \R^{d_2}$ be a sample from the target feature space, characterized by an underlying distribution $\bm{x} \sim p_{\mathrm{data}}(\bm{x})$. Our objective is to train $G_\phi$ such that the distribution of its outputs, $p_\phi(\bm{y})$, aligns to $p_{\mathrm{data}}(\bm{x})$. This alignment can be formulated as maximizing the likelihood of $\bm{y}$ under $p_{\mathrm{data}}$. For instance, in an AE setting where we wish the latent space (from which $\bm{y}$ is sampled) to conform to the distribution of features from a pre-trained feature extractor (from which $\bm{x}$ is sampled), our method facilitates this alignment.

\paragraph{Addressing the Dimension Mismatch} A challenge arises if the latent space dimension $d_1$ differs from the target feature space dimension $d_2$. To address this, we employ fixed (non-learnable) linear projections to map target features $\bm{x}$ from $\R^{d_2}$ to $\R^{d_1}$. We still keep the notation for the projected features and their distribution as $\bm{x}$ and $p_{\mathrm{data}}$ respectively for simplicity. We consider three alternative projection operators: \textit{Random Projection}, \textit{Average Pooling}, and \textit{PCA}. We ablate these methods in Sec.~\ref{sec:ablation} and select random projection as the default due to its simplicity and empirical effectiveness.

The use of linear projection is theoretically supported by the Johnson-Lindenstrauss (JL) lemma~\cite{johnson1984extensions}. The JL lemma states that for a set of $N$ points in $\R^{d_2}$, a random linear mapping can preserve all pairwise Euclidean distances within a multiplicative distortion factor.
%  of $1 \pm \epsilon$, where $\epsilon$ is a small error tolerance and $d_1 \ll d_2$. This property holds with high probability if $d_1 = \mathcal{O}\left(\frac{\log N}{\epsilon^2}\right)$.
The linear projection is defined by a matrix $\bm{W} \in \R^{d_1 \times d_2}$.
Assuming the target features $\bm{x}$ are appropriately normalized, we initialize the projection matrix $\bm{W}$ by sampling its entries from $\mathcal{N}(0, 1/d_2)$. This scaling helps ensure that the components of $\bm{x}$ also have approximately unit variance if the components of $\bm{x}$ are uncorrelated, thereby preserving key statistical properties.

\paragraph{Flow Prior Estimation}
With the projected target features $\bm{x} \sim p_{\mathrm{data}}$, we first train a flow model $\bm{v}_{\theta}: \R^{d_1} \times [0,1] \rightarrow \R^{d_1}$, parameterized by $\theta$. This model is trained using the flow matching objective Eq.~\eqref{eq:loss}, where the dimension $d$ is set to $d_1$, $\bm{x}_0 \sim \mathcal{N}(\bm{0}, \bm{I})$, and $\bm{x}_1$ is replaced by samples $\bm{x}$ from $p_{\mathrm{data}}$. After training, the parameters $\theta$ of the flow model $\bm{v}_\theta$ are frozen. This fixed $\bm{v}_\theta$ implicitly defines a generative process capable of transforming samples from $p_{\mathrm{init}}$ (now in $\R^{d_1}$) into samples that approximate $p_{\mathrm{data}}$. It captures the underlying target distribution and serves as a distributional prior for aligning the latent space.

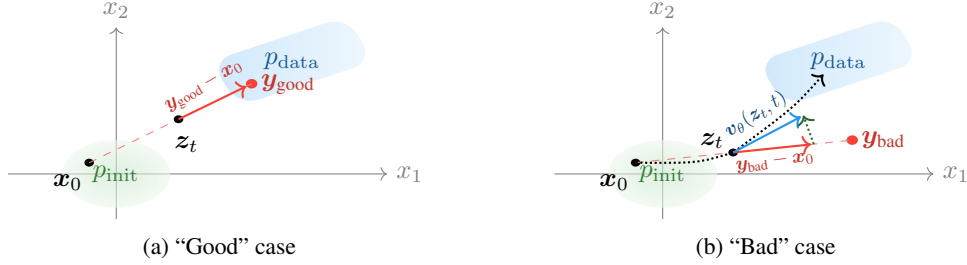
\begin{figure}[htbp]
    \centering
    \definecolor{customgreen}{HTML}{2196f3}
    \definecolor{customblue}{HTML}{4caf50}
    \definecolor{customred}{HTML}{f44336}
    \begin{subfigure}[t]{0.48\textwidth}
        \centering
        \begin{tikzpicture}[xscale=1.8, yscale=1.5, every node/.style={scale=1.0}]
            % Axes
            \draw[->, gray] (-0.8,0) -- (2.0,0) node[right] {$x_1$};
            \draw[->, gray] (0,-0.4) -- (0,1.3) node[above] {$x_2$};
        
            % p_init with flat gradient
            \shade[inner color=customblue!30, outer color=customblue!10, opacity=0.6] (0,0) ellipse (0.4cm and 0.3cm);
            \node[customblue!70!black] at (0,0) {$p_{\mathrm{init}}$};
            \coordinate (x0) at (-0.2, 0.1);
            \fill (x0) circle (1pt) node[below left, shift={(1pt,-1pt)}] {$\bm{x}_0$};
        
            % p_data^proj with gradient
            \shade[left color=customgreen!30, right color=customgreen!10, opacity=0.7, rounded corners=8pt] (0.9,0.6) -- (1.9,1.0) -- (1.7,1.4) -- (0.7, 1.0) -- cycle;
            \node[customgreen!60!black] at (1.3,1.0) {$p_{\mathrm{data}}$};
        
            % y_good (close to p_data^proj)
            \coordinate (y_good) at (1.0, 0.8);
            \fill[customred] (y_good) circle (1.2pt) node[right, shift={(0pt,0pt)}, customred!80!black] {$\bm{y}_{\text{good}}$};
        
            % Straight Path from x0 to y_good
            \draw[dashed, customred!70] (x0) -- (y_good);
            \coordinate (Zt_good) at ($(x0)!0.55!(y_good)$); % Intermediate point
            \fill (Zt_good) circle (1pt) node[below, shift={(3pt,-3pt)}] {$\bm{z}_t$};
        
            % Path Velocity (y_good - x0) vector shown at z_t
            \draw[->, customred, thick, shorten >=1pt] (Zt_good) -- ($(Zt_good)!0.6cm!(y_good)$) node[above, shift={(-0pt,1pt)}, customred!80!black, scale=0.75, midway, sloped] {$\bm{y}_{\text{good}} - \bm{x}_0$};
        
        \end{tikzpicture}
        \caption{``Good'' case}
    \end{subfigure}
    \hfill
    \begin{subfigure}[t]{0.48\textwidth}
        \centering
        \begin{tikzpicture}[xscale=1.8, yscale=1.5, every node/.style={scale=1.0}]
            % Axes
            \draw[->, gray] (-0.8,0) -- (2.0,0) node[right] {$x_1$};
            \draw[->, gray] (0,-0.4) -- (0,1.3) node[above] {$x_2$};
        
            % p_init with flat gradient
            \shade[inner color=customblue!30, outer color=customblue!10, opacity=0.6] (0,0) ellipse (0.4cm and 0.3cm);
            \node[customblue!70!black] at (0,0) {$p_{\mathrm{init}}$};
            \coordinate (x0) at (-0.2, 0.1);
            \fill (x0) circle (1pt) node[below left, shift={(1pt,-1pt)}] {$\bm{x}_0$};
        
            % p_data^proj with gradient
            \shade[left color=customgreen!30, right color=customgreen!10, opacity=0.7, rounded corners=8pt] (0.9,0.6) -- (1.9,1.0) -- (1.7,1.4) -- (0.7, 1.0) -- cycle;
            \node[customgreen!60!black] at (1.3,1.0) {$p_{\mathrm{data}}$};
            \coordinate (flow_actual_target) at (1.2, 0.9); % A point within p_data_proj, target of the flow line
        
            % y_bad (far from p_data^proj)
            \coordinate (y_bad) at (1.4, 0.3);
            \fill[customred] (y_bad) circle (1.2pt) node[right, shift={(0pt,0pt)}, customred!80!black] {$\bm{y}_{\text{bad}}$};
        
            % Straight Path from x0 to y_bad
            \draw[dashed, customred!70] (x0) -- (y_bad);
            \coordinate (Zt_bad) at ($(x0)!0.45!(y_bad)$); % Intermediate point
            \fill (Zt_bad) circle (1pt) node[above right, shift={(-16pt,-2pt)}] {$\bm{z}_t$};
        
            % Path Velocity (y_bad - x0) vector shown at z_t
            \draw[->, customred, thick, shorten >=1pt] (Zt_bad) -- ($(Zt_bad)!0.6cm!(y_bad)$) 
                node[text opacity=1, inner sep=1pt, below=1pt, customred!80!black, scale=0.75, midway, sloped] {$\bm{y}_{\text{bad}} - \bm{x}_0$};

            % Streamline of v_theta starting at z_t, curving towards p_data^proj
            % Control points for two segments
            \coordinate (control1) at ($(x0)!0.5!(Zt_bad) + (0.1,-0.05)$); % Control point for first segment
            \coordinate (control2) at ($(Zt_bad) + (0.45,0.4)$); % Control point for second segment
            \draw[->, black, thick, densely dotted, shorten >=1pt] 
                (x0) .. controls (control1) .. (Zt_bad)
                .. controls (control2) .. (flow_actual_target);

            % Store the endpoints for connecting later
            \coordinate (arrow1_end) at ($(Zt_bad)!0.6cm!(y_bad)$);
            \coordinate (arrow2_end) at ($(Zt_bad)!0.6cm!(control2) + (0.07, -0.07)$);

            % Add new arrow at z_t with same length as y_bad - x_0 but in direction of second segment
            \draw[->, customgreen, thick, shorten >=1pt] (Zt_bad) -- (arrow2_end)
                node[text opacity=1, inner sep=1pt, above=4pt, customgreen!50!black, scale=0.75, midway, sloped] {$\bm{v}_\theta(\bm{z}_t, t)$};

            % Add connecting arrow between the tips
            \draw[->, customblue!60!black, thick, densely dotted] (arrow1_end) -- (arrow2_end);
        \end{tikzpicture}
        \caption{``Bad'' case}
    \end{subfigure}

    \caption{Intuitive illustration of latent space alignment via flow matching, best viewed in color. (a) A ``good'' $\bm{y}_{\text{good}}$ in $p_{\mathrm{data}}$ (green) aligns the straight path velocity (red solid arrow) with the pre-trained flow model's velocity $\bm{v}_\theta(\bm{z}_t,t)$ (overlapped and omitted), yielding low loss. (b) A ``bad'' $\bm{y}_{\text{bad}}$ outside $p_{\mathrm{data}}$ causes a mismatch between the path velocity and $\bm{v}_\theta(\bm{z}_t,t)$ (green solid arrow), resulting in high loss. Minimizing this loss steers $y_{\text{bad}}$ to $p_{\mathrm{data}}$ (blud dotted arrow).}
    \label{fig:intuitive}
    % \vspace{-1em}
\end{figure}

\paragraph{Latent Space Regularization}
Once $\bm{v}_\theta$ is trained and its parameters fixed, we use it to regularize the learnable latents $\bm{y}$. 
The goal is to encourage the distribution $p_\phi(\bm{y})$ to conform to $p_{\mathrm{data}}$ as captured by $\bm{v}_\theta$. 
For each $\bm{y}$ produced by $G_\phi$, we incorporate the flow matching objective described in Eq.~\eqref{eq:loss} into the training objective of $G_\phi$:
\begin{equation}
    \label{eq:loss_y}
    \mathcal{L}_{\text{align}}(\bm{y}; \theta) = \E_{t \sim \U[0, 1], \bm{x}_0 \sim p_{\mathrm{init}}(\bm{x}_0)} \left[ \| \bm{v}_\theta((1-t)\bm{x}_0 + t \bm{y}, t) - (\bm{y} - \bm{x}_0) \|^2 \right]
\end{equation}
Here, $p_{\mathrm{init}}$ is the same $d_1$-dimensional base distribution $\mathcal{N}(\bm{0}, \bm{I})$ used for training $\bm{v}_\theta$. 
% Minimizing $\mathcal{L}_{\text{align}}(\bm{y}; \theta)$ with respect to $\bm{y}$  encourages $\bm{y}$ to be distributed according to the density modeled by $\bm{v}_\theta$.
In Sec.~\ref{sec:proof}, we formally prove that minimizing Eq.~\eqref{eq:loss_y} with respect to $\bm{y}$ serves as a proxy to maximizing a lower bound on the log-likelihood $\log p_1^{\bm{v}_\theta}(\bm{y})$. This establishes that minimizing $\mathcal{L}_{\text{align}}(\bm{y}; \theta)$ effectively trains $G_\phi$ such that its outputs $\bm{y}$ align with the distribution of the target features $\bm{x}$.

The key insight is that the pre-trained velocity field $\bm{v}_\theta$ encapsulates the dynamics that transport probability mass from the base distribution $p_{\mathrm{init}}$ to the target distribution $p_{\mathrm{data}}$ along linear paths. By minimizing $\mathcal{L}_{\text{align}}(\bm{y}; \theta)$, we penalize latents $\bm{y}$ that do not conform to these learned dynamics—that is, $\bm{y}$ values for which the path $(1-t)\bm{x}_0 + t\bm{y}$ is not "natural" under $\bm{v}_\theta$. This procedure shapes $p_\phi(\bm{y})$ to match $p_{\mathrm{data}}$ without requiring explicit computation of potentially intractable likelihoods, relying instead on the computationally efficient flow matching objective.

\subsection{Intuitive Explanation}
\label{sec:intuitive}

Our alignment method leverages the pre-trained flow model, $\bm{v}_\theta$, as an expert on the target feature distribution $p_{\mathrm{data}}$. Having been well trained, $\bm{v}_\theta$ precisely captures the dynamics required to transform initial noise samples $\bm{x}_0$ into target features $\bm{x}$ along straight interpolation paths. Specifically, it has learned to predict the exact velocity $\bm{x} - \bm{x}_0$ at any point $(1-t)\bm{x}_0 + t\bm{x}$ along such a path. This effectively means $\bm{v}_\theta$ can validate whether a given trajectory from noise is characteristic of those leading to the true target distribution.

We utilize this knowledge to shape the distribution of our learnable latents $\bm{y}$. The alignment loss, $\mathcal{L}_{\text{align}}(\bm{y}; \theta)$, challenges $\bm{v}_\theta$: for a given $\bm{y}$ and a random $\bm{x}_0$, it asks whether the velocity field predicted by $\bm{v}_\theta$ along the straight path $(1-t)\bm{x}_0 + t\bm{y}$ matches the path's inherent velocity, $\bm{y} - \bm{x}_0$. If $\bm{y}$ is statistically similar to samples from $p_{\mathrm{data}}$, this match will be close, resulting in a low loss. Conversely, a significant mismatch indicates that $\bm{y}$ is not a plausible target according to the learned dynamics, yielding a high loss. By minimizing this loss (by optimizing the generator $G_\phi$ that produces $\bm{y}$), we iteratively guide $\bm{y}$ towards regions where its connecting path from noise is endorsed by $\bm{v}_\theta$. As depicted in Fig.~\ref{fig:intuitive}, this process progressively aligns the distribution of $\bm{y}$ (blue) with the target distribution $p_{\mathrm{data}}$ (orange), achieving distributional conformity.

\subsection{Relating the Alignment Loss to an ELBO on Log-Likelihood}
\label{sec:proof}

In this section, we demonstrate that minimizing the alignment loss $\mathcal{L}_{\text{align}}(\bm{y}; \theta)$ (Eq.~\eqref{eq:loss_y}) with respect to a given $\bm{y} \in \R^{d_1}$ corresponds to maximizing a variational lower bound (ELBO) on the log-likelihood $\log p_1^{\bm{v}_\theta}(\bm{y})$. Here, $p_1^{\bm{v}_\theta}(\bm{y})$ denotes the probability density at $t=1$ induced by the ODE dynamics $\frac{\mathrm{d}\bm{z}_t}{\mathrm{d}t} = \bm{v}_\theta(\bm{z}_t, t)$, with $\bm{z}_0 \sim p_{\mathrm{init}}$.

\begin{proposition}
Let $\bm{v}_\theta: \R^{d_1} \times [0,1] \to \R^{d_1}$ be a given velocity field, and $p_{\mathrm{init}}$ be a base distribution. For $\bm{y} \in \R^{d_1}$, the log-likelihood $\log p_1^{\bm{v}_\theta}(\bm{y})$ is lower-bounded as:
\begin{equation}
    \label{eq:prop_elbo_statement}
\log p_1^{\bm{v}_\theta}(\bm{y}) \ge C(\bm{y}) -  \lambda \mathcal{L}_{\text{align}}(\bm{y}; \theta),
\end{equation}
where $\lambda > 0$ is a constant, $\mathcal{L}_{\text{align}}(\bm{y}; \theta)$ is defined in Eq.~\eqref{eq:loss_y}, and $C(\bm{y})$ is dependent on $\bm{y}$ and $\bm{v}_\theta$.
\end{proposition}

\begin{proof}
We establish this result by constructing a specific variational lower bound on $\log p_1^{\bm{v}_\theta}(\bm{y})$. Variational lower bounds for log-likelihoods in continuous-time generative models can be constructed by introducing a proposal distribution for the latents that could generate $\bm{y}$. Consider a family of "proposal" paths~\cite{lipman2022flow}, which are straight lines interpolating from an initial point $\bm{x}_0 \sim p_{\mathrm{init}}$ to the given point $\bm{y}$:
\begin{equation}
    \label{eq:proof_straight_path_def}
    \bm{z}_s(\bm{x}_0, \bm{y}) = (1-s)\bm{x}_0 + s\bm{y}, \quad s \in [0,1]
\end{equation}
The velocity of such a path is constant: $\dot{\bm{z}}_s(\bm{x}_0, \bm{y}) = \bm{y} - \bm{x}_0$.
We adopt a variational distribution over the initial states $\bm{x}_0$, conditioned on $\bm{y}$, as $q(\bm{x}_0 | \bm{y}) = p_{\mathrm{init}}(\bm{x}_0)$. That is, we consider initial states drawn from the prior, irrespective of $\bm{y}$ for the functional form of $q$.

A known variational lower bound on $\log p_1^{\bm{v}_\theta}(\bm{y})$s~\cite{chen2018neural,grathwohl2018ffjord,liu2022flow}
 can be written a:
\begin{align}
    \label{eq:proof_elbo_form_general}
    \log p_1^{\bm{v}_\theta}(\bm{y}) \ge \mathbb{E}_{\bm{x}_0 \sim q(\cdot|\bm{y})} \Bigg[ & \log p_{\mathrm{init}}(\bm{x}_0) - \int_0^1 \left( \lambda_s \| \dot{\bm{z}}_s(\bm{x}_0, \bm{y}) - \bm{v}_\theta(\bm{z}_s(\bm{x}_0, \bm{y}), s) \|^2 \right) \mathrm{d}s \nonumber \\
    & - \log q(\bm{x}_0|\bm{y}) - \int_0^1 \left( \mathrm{Tr}(\nabla_{\bm{z}_s} \bm{v}_\theta(\bm{z}_s(\bm{x}_0, \bm{y}), s)) \right) \mathrm{d}s \Bigg]
\end{align}
Here, $\lambda_s > 0$ is a time-dependent weighting factor. For simplicity and consistency with the definition of $\mathcal{L}_{\text{align}}$ (Eq.~\eqref{eq:loss_y}), we set $\lambda_s = \lambda = 1$ for all $s \in [0,1]$.
With $q(\bm{x}_0|\bm{y}) = p_{\mathrm{init}}(\bm{x}_0)$, the term $\log p_{\mathrm{init}}(\bm{x}_0) - \log q(\bm{x}_0|\bm{y})$ vanishes.
Substituting the expressions for $\bm{z}_s(\bm{x}_0, \bm{y})$ from Eq.~\eqref{eq:proof_straight_path_def} and its velocity $\dot{\bm{z}}_s(\bm{x}_0, \bm{y}) = \bm{y} - \bm{x}_0$:
\begin{align}
\label{eq:proof_elbo_final_substituted}
\log p_1^{\bm{v}_\theta}(\bm{y}) \ge & - \mathbb{E}_{\bm{x}_0 \sim p_{\mathrm{init}}} \left[ \int_0^1 \mathrm{Tr}(\nabla_{\bm{z}} \bm{v}_\theta((1-s)\bm{x}_0 + s\bm{y}, s)) \mathrm{d}s \right] \nonumber \\
& - \mathbb{E}_{\bm{x}_0 \sim p_{\mathrm{init}}} \left[ \int_0^1 \| (\bm{y} - \bm{x}_0) - \bm{v}_\theta((1-s)\bm{x}_0 + s\bm{y}, s) \|^2 \mathrm{d}s \right]
\end{align}
The second term in this inequality matches the definition of $\mathcal{L}_{\text{align}}(\bm{y}; \theta)$ (Eq.~\eqref{eq:loss_y}).
Let us define the first term of the ELBO's right-hand side. To maintain consistency with the expectation over time in $\mathcal{L}_{\text{align}}$, we can write:
\begin{equation}
\label{eq:proof_C_y_def}
C(\bm{y}) = - \mathbb{E}_{s \sim \mathcal{U}[0,1], \bm{x}_0 \sim p_{\mathrm{init}}} \left[ \mathrm{Tr}(\nabla_{\bm{z}} \bm{v}_\theta((1-s)\bm{x}_0 + s\bm{y}, s)) \right]
\end{equation}
So, the ELBO (Eq.~\eqref{eq:proof_elbo_final_substituted}) can be expressed as:
\begin{equation}
\label{eq:proof_elbo_simplified_final}
\log p_1^{\bm{v}_\theta}(\bm{y}) \ge C(\bm{y}) - \mathcal{L}_{\text{align}}(\bm{y}; \theta)
\end{equation}
This concludes the derivation of the lower bound as stated in the proposition (with $\lambda=1$).
\end{proof}

\paragraph{Interpretation and Significance}
The inequality \eqref{eq:proof_elbo_simplified_final} demonstrates that maximizing the derived lower bound with respect to $\bm{y}$ involves two parts: maximizing $C(\bm{y})$ and minimizing $\mathcal{L}_{\text{align}}(\bm{y}; \theta)$.
The term $\mathcal{L}_{\text{align}}(\bm{y}; \theta)$ directly measures how well the velocity field $\bm{v}_\theta$ predicts the velocity of that straight path, i.e., $\bm{y}-\bm{x}_0$. Minimizing this term forces $\bm{y}$ into regions where it behaves like a point reachable from $p_{\mathrm{init}}$ via a straight path whose dynamics are consistent with the learned $\bm{v}_\theta$. This is precisely the behavior expected if $\bm{y}$ were a sample from the distribution $p_{\mathrm{data}}$.

The term $C(\bm{y})$ represents the expected negative trace of the Jacobian of $\bm{v}_\theta$, averaged over the chosen straight variational paths. By minimizing $\mathcal{L}_{\text{align}}(\bm{y}; \theta)$, we are not strictly maximizing the ELBO in Eq.~\eqref{eq:proof_elbo_simplified_final} with respect to $\bm{y}$. Instead, we are optimizing a crucial component of it that directly enforces consistency with the learned flow dynamics. We analyze the behavior of $C(\bm{y})$ in Appendix~\ref{app:proof} to show that if $\bm{y}$ aligns with a more concentrated target distribution (making $\mathcal{L}_{\text{align}}(\bm{y}; \theta)$ small), $C(\bm{y})$ tends to be positive and larger, contributing favorably to the ELBO.
\begin{assumption}[Optimality of $\bm{v}_\theta$]
    \label{ass:vtheta_fixed}
    The velocity field $\bm{v}_\theta: \R^{d_1} \times [0,1] \to \R^{d_1}$ is (pre-trained) and optimal, satisfying
        $ \bm{v}_\theta((1-t)\bm{x}_0 + t\bm{x}_1, t) = \bm{x}_1 - \bm{x}_0 \quad \forall \bm{x}_0 \sim p_{\mathrm{init}}, \bm{x}_1 \sim p_{\mathrm{data}}, t \in [0,1]$.
\end{assumption}
To further interpret the method, we consider the Assumption~\ref{ass:vtheta_fixed} that $\bm{v}_\theta$ is optimally trained such that $\bm{v}_\theta((1-s)\bm{x}_0 + s\bm{x}_1, s) = \bm{x}_1 - \bm{x}_0$ for $\bm{x}_1 \sim p_{\mathrm{data}}$. If $\bm{y}$ is itself a sample from $p_{\mathrm{data}}$, then $\mathcal{L}_{\text{align}}(\bm{y}; \theta)$ would be (close to) zero. However, when optimizing an arbitrary $\bm{y}$, especially if it is far from $p_{\mathrm{data}}$, the $\mathcal{L}_{\text{align}}(\bm{y}; \theta)$ term can be substantial. Its minimization drives $\bm{y}$ towards regions of higher plausibility under the learned flow.

In practice, directly maximizing $\log p_1^{\bm{v}_\theta}(\bm{y})$ via Eq.~\eqref{eq:likelihood} is computationally demanding, requiring ODE solves and computation of Jacobian traces along these true ODE paths. Maximizing the full ELBO (Eq.~\eqref{eq:prop_elbo_statement}) would still require computing $C(\bm{y})$, which involves trace computations. By focusing on minimizing only $\mathcal{L}_{\text{align}}(\bm{y}; \theta)$, we adopt a computationally tractable proxy. This objective encourages $\bm{y}$ to have a high likelihood under $p_1^{\bm{v}_\theta}(\bm{y})$ by ensuring consistency with the learned flow dynamics, thereby aligning the distribution of $\bm{y}$ with the target distribution $p_{\mathrm{data}}$ implicitly modeled by $\bm{v}_\theta$. A more complete proof can be found in Appendix~\ref{app:proof}.

% This completes the formal argument that minimizing $\mathcal{L}_{\text{align}}(\bm{y}; \theta)$ with respect to $\bm{y}$ corresponds to maximizing a variational lower bound on $\log p_1^{\bm{v}_\theta}(\bm{y})$.
% (related to path integral formulations or specific variational inference schemes for continuous-time generative models, e.g., similar to those discussed in the context of score-based models or variational autoencoders for sequential data, though often presented for SDEs and adaptable to ODEs with careful formulation)
% \begin{align}
% \label{eq:proof_elbo_form_general}
% \log p_1^{\bm{v}_\theta}(\bm{y}) \ge \mathbb{E}_{\bm{x}_0 \sim q(\cdot|\bm{y})} \Bigg[ & \log p_{\mathrm{init}}(\bm{x}_0) - \log q(\bm{x}_0|\bm{y}) \nonumber \\
% & - \int_0^1 \left( \lambda_s \| \dot{\bm{z}}_s(\bm{x}_0, \bm{y}) - \bm{v}_\theta(\bm{z}_s(\bm{x}_0, \bm{y}), s) \|^2 + \mathrm{Tr}(\nabla_{\bm{z}_s} \bm{v}_\theta(\bm{z}_s(\bm{x}_0, \bm{y}), s)) \right) \mathrm{d}s \Bigg].
% \end{align}

\section{Experiments}

This section presents an empirical validation of the proposed alignment method with flow priors. The investigation starts with an illustrative experiment in Sec.~\ref{sec:toy}. Subsequently, large-scale experiments are conducted on image generation tasks using the ImageNet dataset, as detailed in Sec.~\ref{sec:generation}. In Sec.~\ref{sec:ablation}, we conduct ablation studies of the proposed method.

\begin{figure}
    \begin{minipage}{0.3\textwidth}
        \centering
        {\small (a)}
        \includegraphics[width=\linewidth]{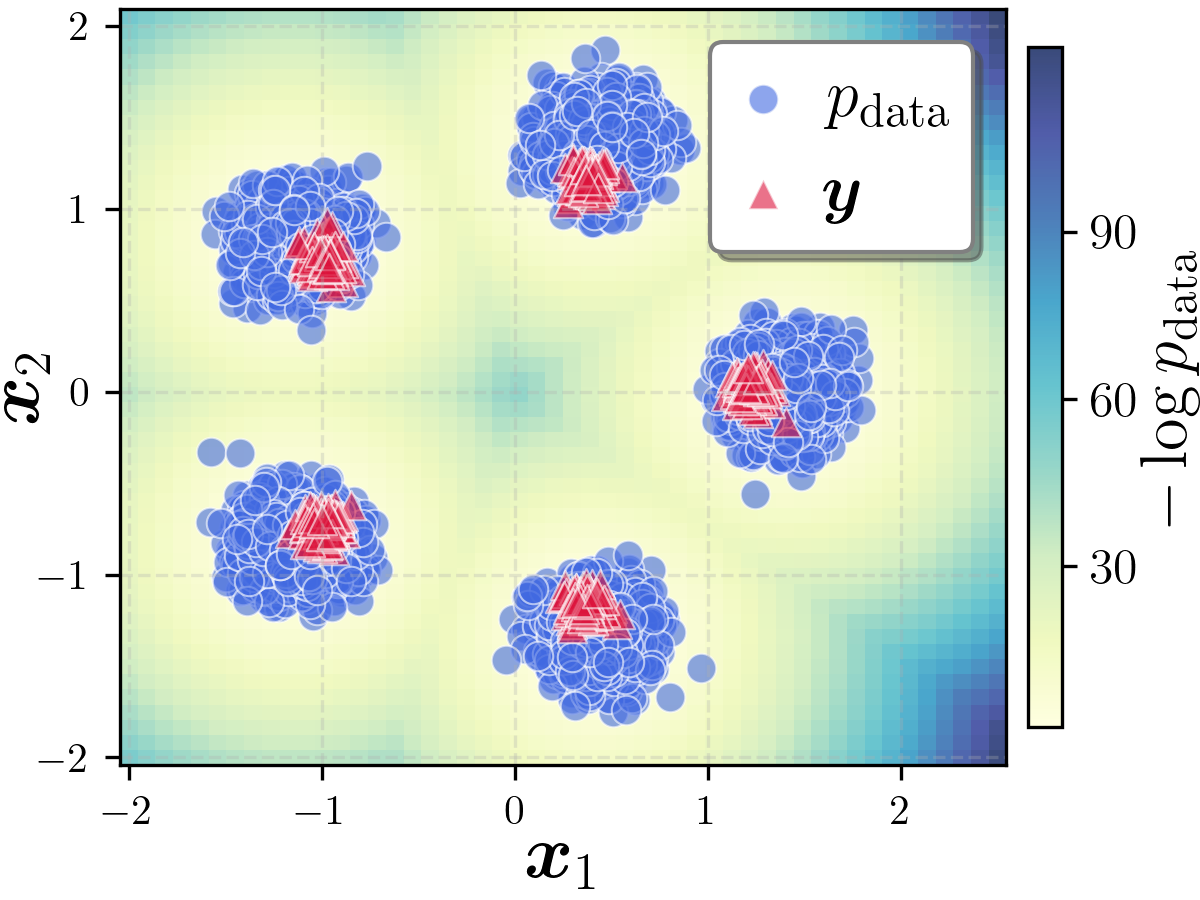}
    \end{minipage}
    \hfill
    \begin{minipage}{0.3\textwidth}
        \centering
        {\small (b)}
        \includegraphics[width=\linewidth]{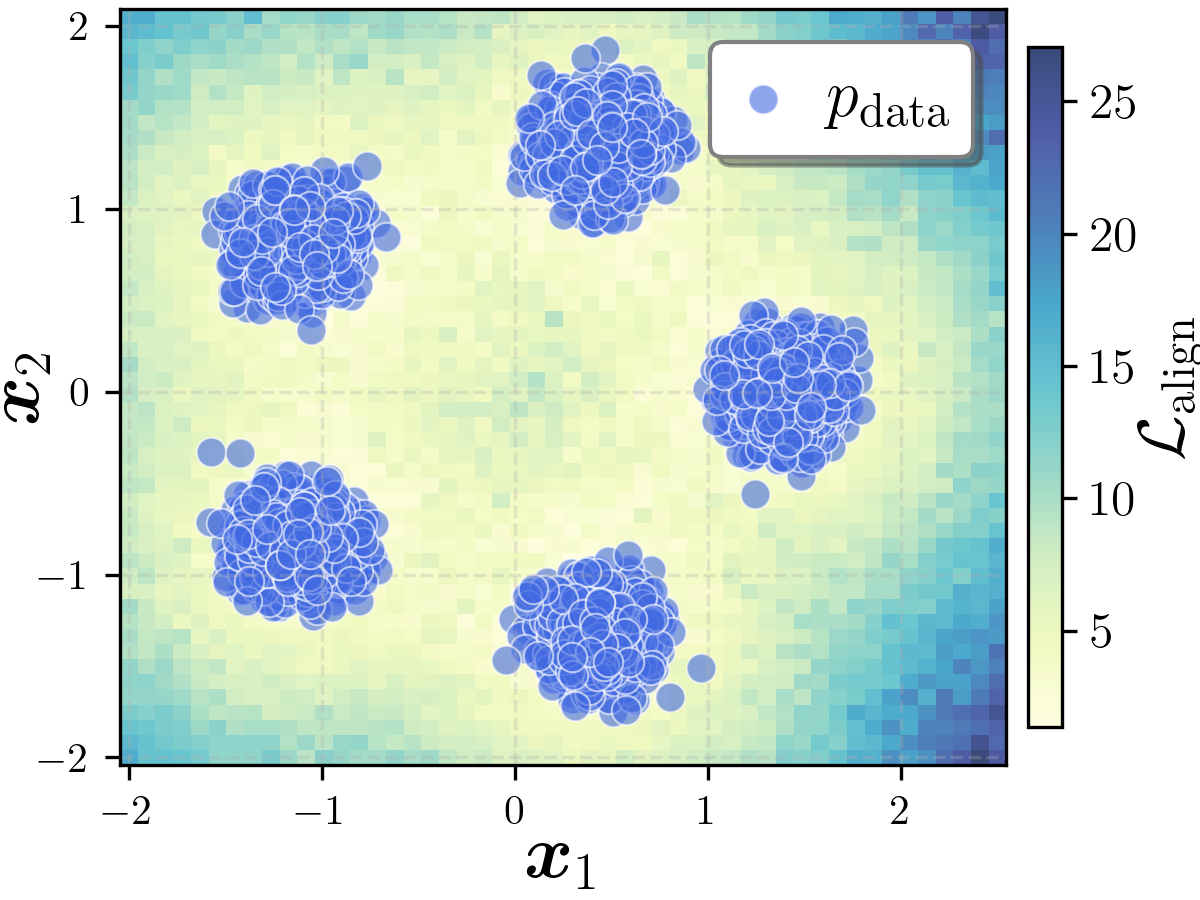}
    \end{minipage}
    \hfill
    \begin{minipage}{0.3\textwidth}
        \centering
        {\small (c)}
        \includegraphics[width=\linewidth]{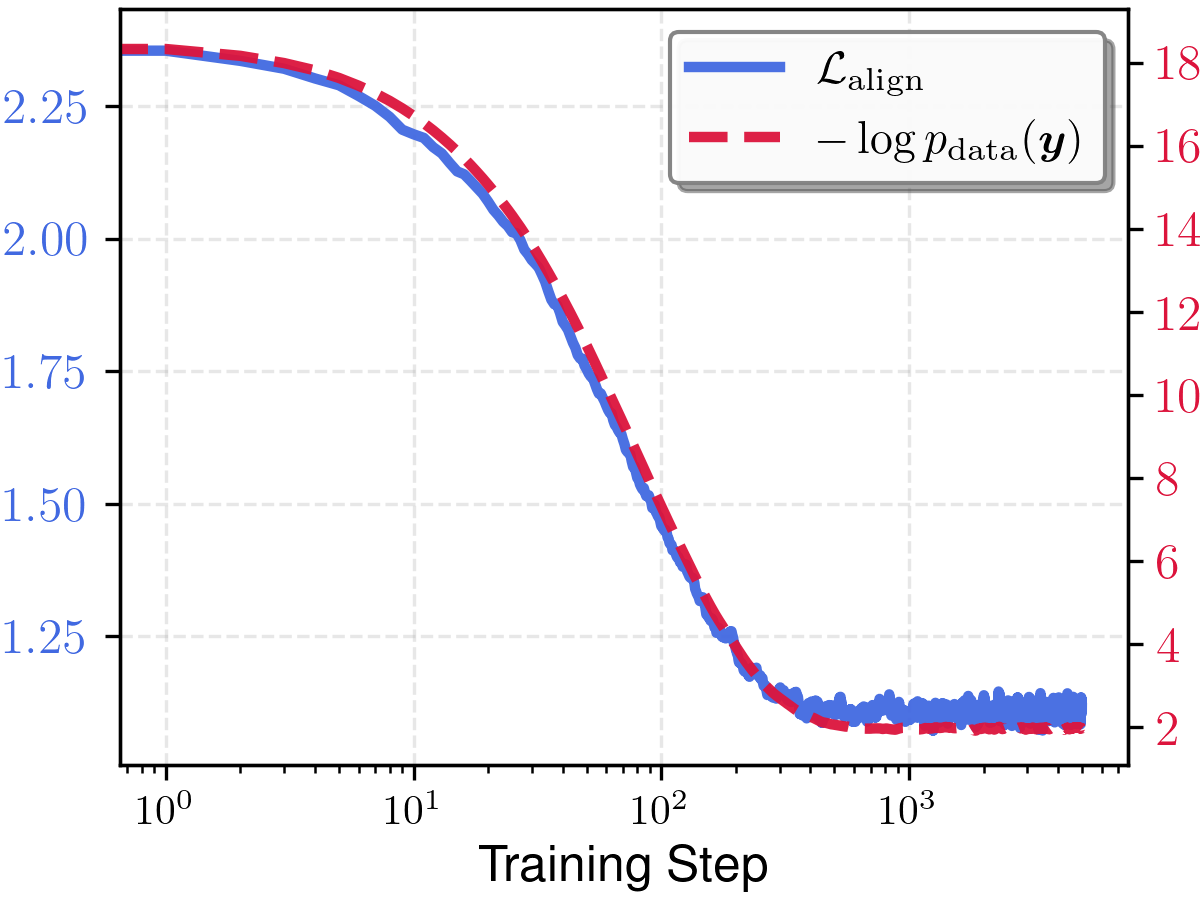}
    \end{minipage}
    \caption{Illustration with a Mixture of Gaussians distribution. (a) Aligned latent variables $\bm{y}$ (red triangles) concentrate in low negative log-likelihood (NLL) regions of $p_\text{data}$ (blue dots; heatmap shows $-\log p_\text{data}$). (b) Alignment loss $\mathcal{L}_{\text{align}}$ heatmap mirrors the NLL landscape of $p_\text{data}$, with $p_\text{data}$ samples in low-$\mathcal{L}_{\text{align}}$ areas. (c) $\mathcal{L}_{\text{align}}$ (blue solid) and $-\log p_\text{data}(\bm{y})$ (red dashed) decline simultaneously in training, showing $\mathcal{L}_{\text{align}}$ serves as a proxy for maximizing the log-likelihood of $\bm{y}$ under $p_\text{data}$.}
    \label{fig:toy_example}
    \vspace{-1em}
\end{figure}

\subsection{Toy Examples}
\label{sec:toy}

We present a toy example as an illustrative experiment in a 2D setting. The target distribution, denoted $p_{\text{data}}$, is configured as a mixture of five isotropic Gaussians. Following the methodology outlined in Sec.~\ref{sec:pipeline}, we first train a flow model $\bm{v}_\theta$ to map a standard 2D Normal distribution $\mathcal{N}(\bm{0}, \bm{I})$ to $p_{\text{data}}$. This flow model $\bm{v}_\theta$ is implemented by a multi-layer perceptron (MLP) incorporating adaptive layer normalization for time modulation~\cite{Peebles2022ScalableDM}. Upon completion of training, the parameters $\theta$ of this flow model are frozen.  Subsequently, instead of a parameterized model $G_\phi$, we directly initialize a set of learnable 2D variables as $\bm{y}$ and optimize them by minimizing the alignment loss $\mathcal{L}_{\text{align}}(\bm{y}; \theta)$.

The results are presented in Fig.~\ref{fig:toy_example}.
Fig.~\ref{fig:toy_example} (a) compares the target distribution $p_{\text{data}}$ (blue point samples) with the optimized variables $\bm{y}$ (red triangles). The background visualizes the negative log-likelihood (NLL) of $p_{\text{data}}$, which is computed analytically. It is evident that $\bm{y}$ successfully converges to the high log-likelihood regions of $p_{\text{data}}$.
Fig.~\ref{fig:toy_example} (b) displays the landscape of the alignment loss $\mathcal{L}_{\text{align}}$, which is estimated numerically with $\bm{v}_\theta$. The landscape mirrors the NLL surface of $p_{\text{data}}$ depicted in (a). Samples drawn from $p_{\text{data}}$ (blue dots) are concentrated in regions where $\mathcal{L}_{\text{align}}$ is low, suggesting that $\mathcal{L}_{\text{align}}$ effectively captures the underlying structure of the target distribution.
Fig.~\ref{fig:toy_example} (c) illustrates $\mathcal{L}_{\text{align}}$ (blue solid line) and the true NLL $-\log p_{\text{data}}(\bm{y})$ (red dashed line) during the training of $\bm{y}$. The alignment loss and the NLL exhibit a strong positive correlation, decreasing concomitantly throughout the training process.
More detailed toy examples can be found in Appendix~\ref{app:toy}.

% More toy examples with different target distributions and more detailed visualizations can be found in the Appendix.

% Collectively, these observations (Fig.~\ref{fig:toy_example}a-c) provide empirical evidence that minimizing $\mathcal{L}_{\text{align}}$ serves as a computationally tractable and effective proxy for maximizing the log-likelihood of $\bm{y}$ under the target distribution $p_{\text{data}}$. More toy examples with different target distributions and more detailed visualizations can be found in the Appendix.

\subsection{Image Generation}
\label{sec:generation}

Prior work has demonstrated that aligning the latent space of AEs with semantic encoders can enhance generative model performance~\cite{chen2025masked,yao2025reconstruction,chen2024softvq,qiu2025robust}.
To validate this observation and further showcase the efficacy of our proposed method, we conduct large-scale image generation experiments on the ImageNet-1K~\cite{deng2009imagenet} dataset at $256 \times 256$ resolution.

\paragraph{Implementation Details} 
\label{sec:implementation}
Our AE architecture employs two Vision Transformer (ViT)-Large~\cite{Dosovitskiy2020AnII} models, each with 391M parameters, serving as the latent encoder and decoder, respectively. 
The encoder maps input images to a latent space of 64 tokens, each with dimension 32, striking a balance between reconstruction quality and computational efficiency. We impose \textit{token-level} alignment on the latents. 
The alignment loss on the latents is set to $\lambda = 0.01$ by default. We also incorporate conventional reconstruction loss, perceptual loss, and adversarial loss on the pixel outputs~\cite{Esser2020TamingTF,rombach2022high}.

% For the target distribution $p_{\text{data}}$, we study 4 different variants including pre-trained VAE latents, continuous visual features, discrete visual features, and text embeddings. The flow prior is implemented as a 6-layer MLP with 1024 hidden units and GELU activations, trained for 1M steps using the Adam optimizer. The autoencoders are trained for 200k steps unless otherwise specified.
For the target distribution $p_{\text{data}}$, we investigate four distinct variants:  \textit{low-level} visual features from a VAE~\cite{kingma2013auto,Esser2020TamingTF}, continuous \textit{semantic} visual features from DinoV2~\cite{oquab2023dinov2}, \textit{discrete} visual codebook embeddings from LlamaGen VQ~\cite{sun2024autoregressive,van2017neural}, and \textit{textual} embeddings from Qwen~\cite{bai2023qwen}. Their feature dimensions are 32, 768, 8, 896, respectively.
The flow-based prior is modeled by a 6-layer MLP with 1024 hidden units, trained for 1 million steps using the AdamW~\cite{loshchilov2017decoupled} optimizer to match Assumption~\ref{ass:vtheta_fixed}. More details can be found in Appendix~\ref{app:implementation}.
% We evaluate three state-of-the-art generative architectures: MAR-B, DiT-L, and SiT-L. For MAR-B, we add qk-norm and replace the diffusion head with a flow head for stable training. All models are trained for 100k steps. To enable apples-to-apples comparison, we keep identical configurations and hardware across all variants except for the tokenizer. More implementation details can be found in the Appendix.
\paragraph{Alignment Results}
% We first demonstrate the effectiveness of the alignment. 
Analogous to the toy example, we aim to correlate the alignment loss $\mathcal{L}_{\text{align}}$ with the NLL of latents under the target distribution $p_{\text{data}}$. 
Since the NLL is intractable for implicitly defined distributions, we estimate the density using $k$-nearest neighbors. 
The probability density $p(\bm{y})$ at a point $\bm{y}$ is inversely proportional to the volume of the hypersphere enclosing its $k^{\text{th}}$ nearest neighbor among the target samples. 
Consequently, the NLL can be estimated as $- \log p_{\text{data}}(\bm{y}) \propto \log r_k(\bm{y})$ where $r_k(\bm{y})$ is the Euclidean distance to the $k^{\text{th}}$ neighbor, and $D$ is the dimension.
We use $\log r_k(\bm{y})$ as our proxy measure for the NLL.
We first index the set of target distribution samples using Faiss~\cite{douze2024faiss}. 
During the training, we periodically sample 10k points from the latent space and measure the alignment quality by averaging the $\log r_k(\bm{y})$. 

The results are presented in Fig.~\ref{fig:alignment}. 
A strong correlation is observed between the alignment loss $\mathcal{L}_{\text{align}}$ and the $k$-NN distance proxy $\log r_k(\bm{y})$.
The only unstable case is the VQ variant, for which the GAN loss collapses during training due to its low dimension (8-dim), yet the general trend is still consistent.
This finding corroborates our conclusion that $\mathcal{L}_{\text{align}}$ serves as an effective proxy for the NLL of the latents under $p_{\text{data}}$.
Crucially, our method captures the underlying structure across diverse target distributions, spanning different forms (continuous, discrete) and modalities (visual, textual), even when applied to a large-scale dataset like ImageNet and a high-capacity model such as ViT-Large.
\begin{figure}
    \begin{minipage}{0.24\textwidth}
        \centering
        {\small VAE (Low-level)}
        \includegraphics[width=\linewidth]{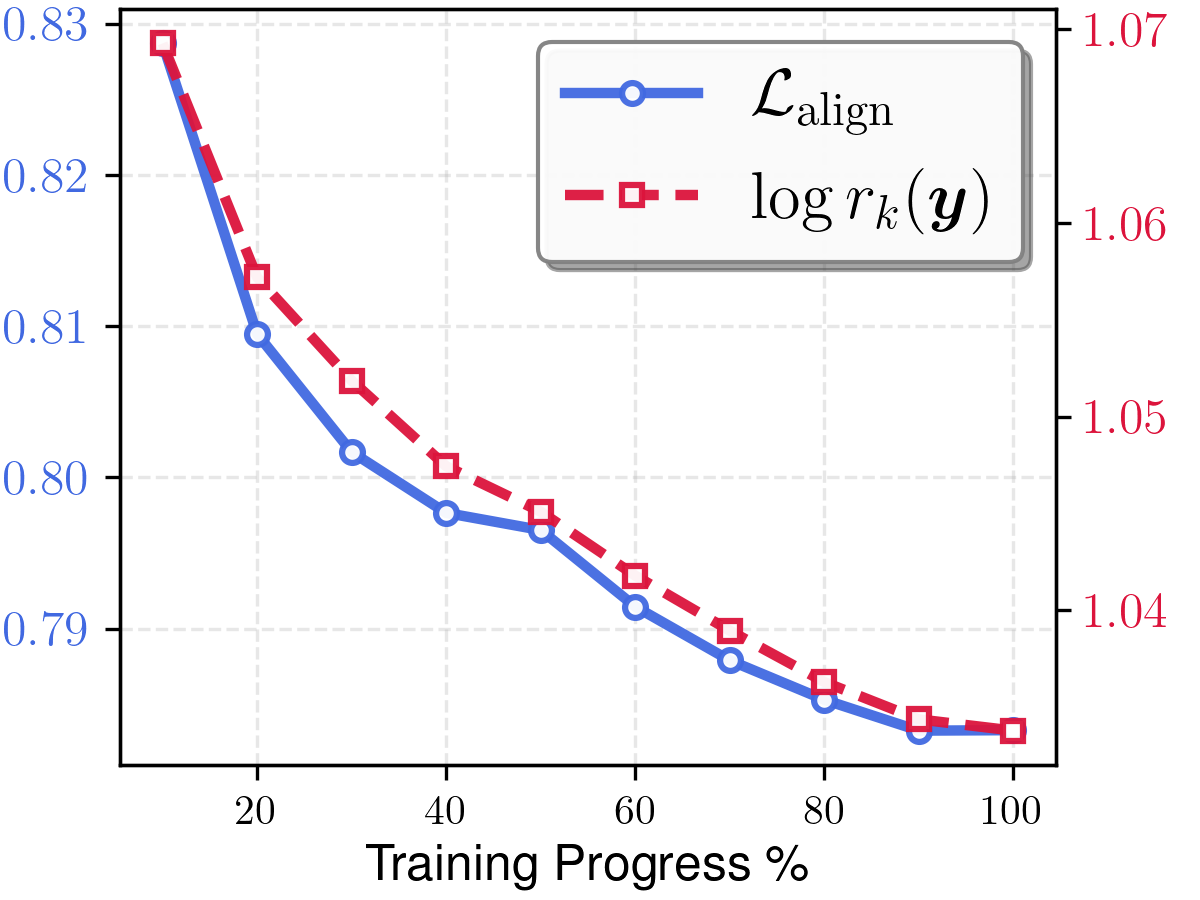}
    \end{minipage}
    \hfill
    \begin{minipage}{0.24\textwidth}
        \centering
        {\small DinoV2 (Semantic)}
        \includegraphics[width=\linewidth]{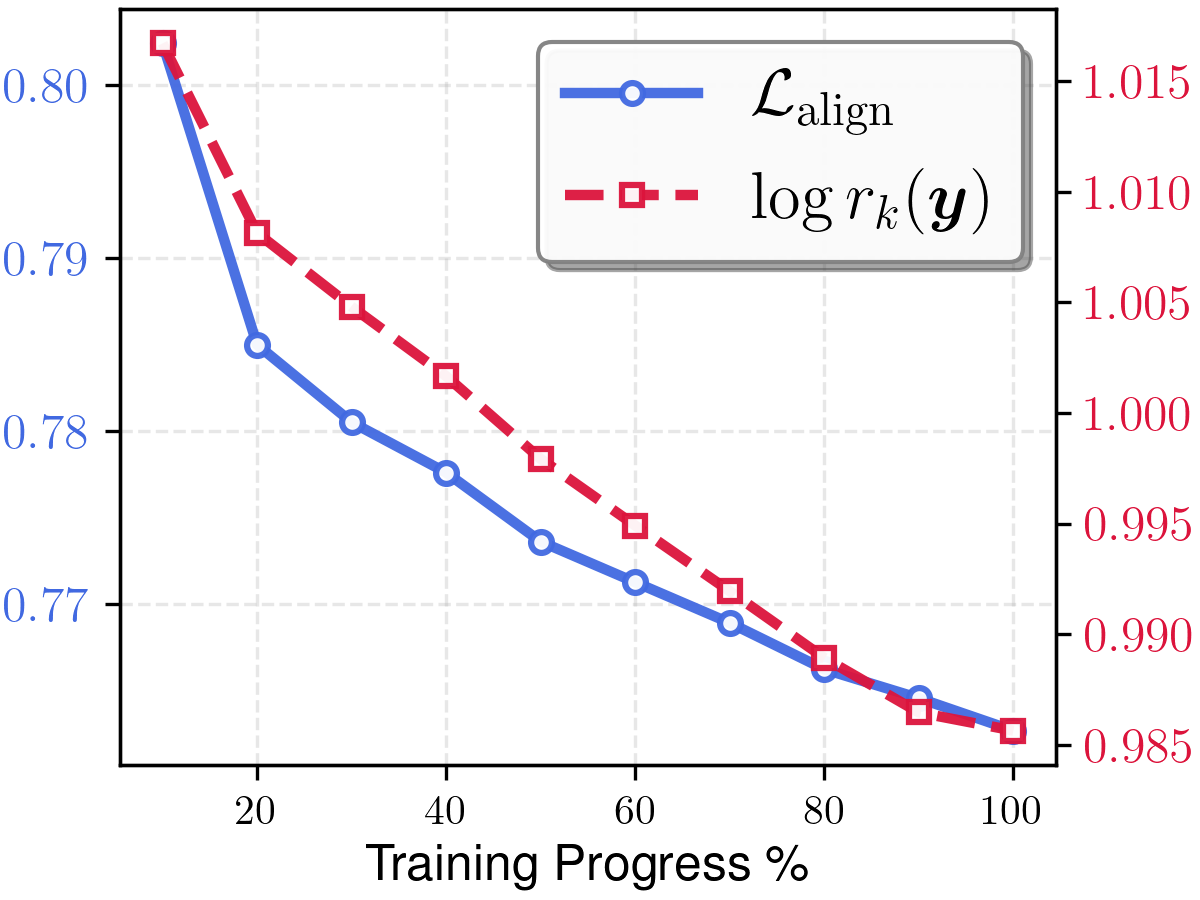}
    \end{minipage}
    \hfill
    \begin{minipage}{0.24\textwidth}
        \centering
        {\small VQ (Discrete)}
        \includegraphics[width=\linewidth]{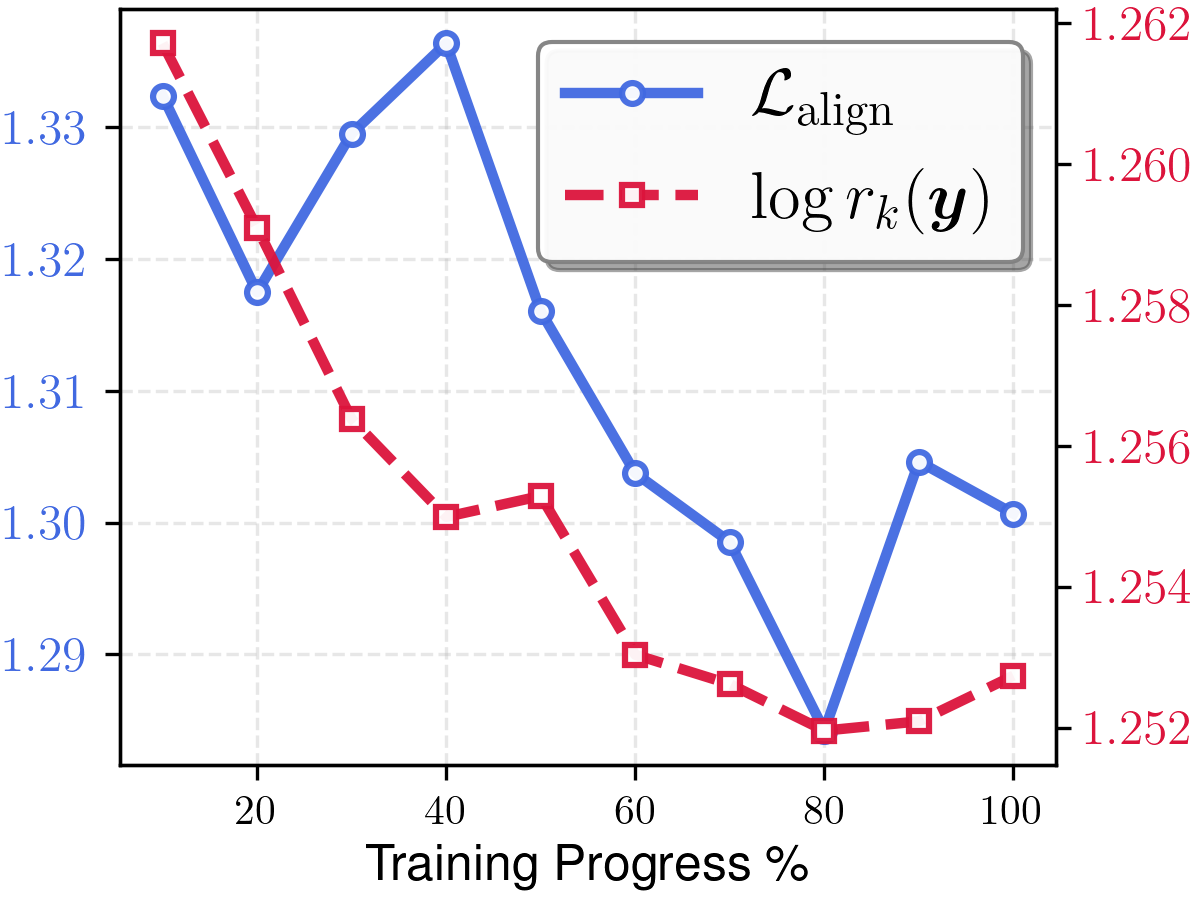}
    \end{minipage}
    \hfill
    \begin{minipage}{0.24\textwidth}
        \centering
        {\small Qwen (Textual)}
        \includegraphics[width=\linewidth]{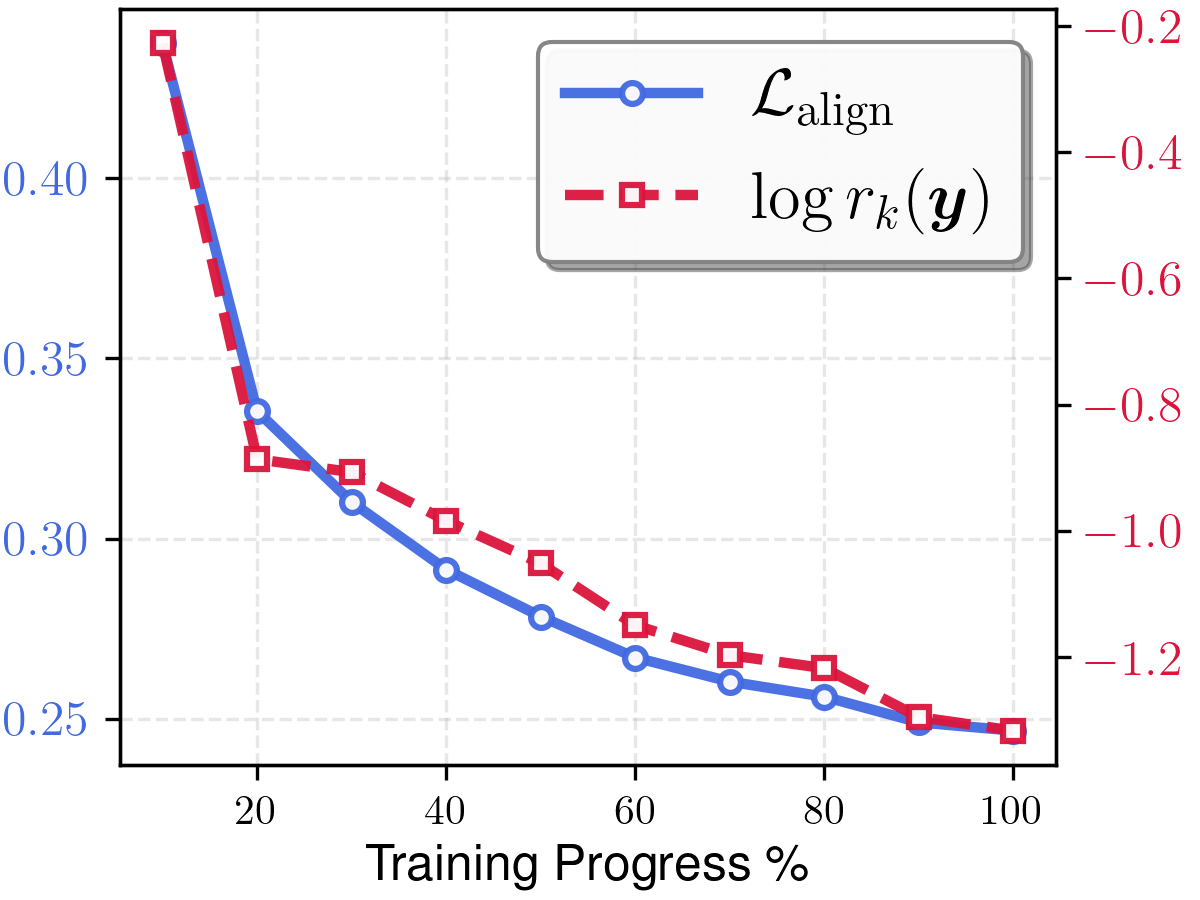}
    \end{minipage}
    \caption{Aligning autoencoders on ImageNet-1K with different target distributions. The alignment loss $\mathcal{L}_{\text{align}}$ (blue solid) and the $k$-NN distance $\log r_k(\bm{y})$ (red dashed) are proportional throughout the training. Confirming that $\mathcal{L}_{\text{align}}$ serves as a good proxy for the NLL of the latents under $p_{\text{data}}$.}
    \label{fig:alignment}
    \vspace{-1em}
\end{figure}
\paragraph{Generation Results}
After demonstrating effective latent space alignment, we investigated its impact on generative model performance. 
We evaluated both reconstruction and generation capabilities on ImageNet using the MAR-B~\cite{li2024autoregressive} architecture. 
For MAR-B, we incorporated qk-norm~\cite{dehghani2023scaling} and replaced the diffusion head with a flow head to ensure stable training. 
We choose flow-based MAR-B as it does not favor continuous Gaussian-like latent structure like Diffusion models~\cite{song2020denoising,dhariwal2021diffusion,Karras2022ElucidatingTD,Nichol2021ImprovedDD,rombach2022high} do.
% All models were trained for 100k steps with a base learning rate of $1 \times 10^{-4}$ using the AdamW optimizer. 
To ensure an ``apple-to-apple'' comparison, configurations and hardware remained identical across all experiments, with the only difference being the specific AE used for each alignment variant.
\begin{table*}[htbp]
    \centering
    \footnotesize
    \setlength{\tabcolsep}{4pt}  % Reduce column separation
    \caption{ImageNet $256\times256$ conditional generation using MAR-B. All models are trained and evaluated using identical settings. The CFG scale is tuned for KL and kept the same for others.}
    \label{tab:imagenet}
    \begin{tabular}{l|cc|cccc|cccc}
        \toprule
         \multirow{2}{*}{Autoencoder} & \multirow{2}{*}{rFID$\downarrow$} & \multirow{2}{*}{PSNR$\uparrow$} & \multicolumn{4}{c|}{w/o CFG} & \multicolumn{4}{c}{w/ CFG} \\
         & & & FID$\downarrow$ & IS$\uparrow$ & Pre.$\uparrow$ & Rec.$\uparrow$ & FID$\downarrow$ & IS$\uparrow$ & Pre.$\uparrow$ & Rec.$\uparrow$ \\
        \midrule
    AE              & 1.13 & 20.20 & 15.08 & 86.37 & \textbf{0.60} & \textbf{0.59} & 5.26 & 237.60 & 0.56 & 0.65 \\
    KL              & 1.65 & 22.59 & 12.94 & 91.86 & 0.60 & 0.58 & 5.29 & 200.85 & 0.57 & 0.65 \\
    SoftVQ          & \textbf{0.61} & 23.00 & 13.30 & 93.40 & 0.60 & 0.57 & 6.09 & 198.53 & \textbf{0.58} & 0.61 \\
    \midrule
    Low-level (VAE) & 1.22 & 22.31 & 12.04 & 98.66 & 0.56 & 0.57 & 5.02 & 240.03 & 0.56 & 0.62 \\
    Semantic (Dino) & 1.26 & 23.07 & \textbf{11.47} & 101.74 & 0.59 & 0.59 & \textbf{4.87} & \textbf{250.38} & 0.54 & 0.67 \\
    Discrete (VQ)   & 2.99 & 22.32 & 24.63 & 48.17 & 0.55 & 0.53 & 10.04 & 119.64 & 0.47 & 0.65 \\
    Textual (Qwen)  & 0.85 & \textbf{23.12} & 11.89 & \textbf{102.23} & 0.55 & 0.57 & 6.56 & 262.89 & 0.49 & \textbf{0.69}  \\
    \bottomrule
    \end{tabular}
    \vspace{-1em}
\end{table*}

The results are presented in Tab.~\ref{tab:imagenet}. 
Reconstruction performance was measured by rFID~\cite{Heusel2017GANsTB} and PSNR on the ImageNet validation set. 
Generation performance was assessed using FID, IS~\cite{salimans2016improved}, Precision, and Recall on 50k generated samples and the validation set, both with and without classifier-free guidance (CFG)~\cite{ho2022classifier}. 
Our key findings are:

\textit{1) Alignment vs. Reconstruction Trade-off:} Latent space alignment typically degrades reconstruction quality (rFID, PSNR) compared to vanilla AEs, as constraints reduce capacity. SoftVQ\cite{chen2024softvq} excels among aligned methods due to its sample-level alignment.  \textit{2) Alignment Enhances Generation:} Structured latent spaces improve generative metrics (FID, IS), but complexity is not decisive. Simple features (text embeddings like Qwen) may match the performance of richer visual features (DinoV2).  \textit{3) Optimal prior selection is open:} No consensus exists on optimal priors. Low-dimensional discrete features (LlamaGen VQ) underperform, while cross-modal alignment (Qwen text embeddings) demonstrates transferable structural benefits. More discussions can be found in Appendix~\ref{app:exp}.

\subsection{Ablation Study}
\label{sec:ablation}

\begin{table}[t]
    \centering
    \caption{Ablation studies on ImageNet $256\times 256$ for different configurations using autencoders regularized by textual features (Qwen). We use a shorter training schedule when ablating weight.}
    \vspace{1mm}
    \footnotesize

    \begin{subtable}[t]{0.48\textwidth}
    \centering
    \caption{Downsampling Methods}
    \begin{tabular}{@{}lccccc@{}}
    \toprule
    Method & rFID$\downarrow$ & PSNR$\uparrow$ & FID$\downarrow$ & IS$\uparrow$ \\
    \midrule
    Random Proj. & \textbf{0.85} & 23.12 & \textbf{11.89} & \textbf{102.23} \\
    Avg. Pooling & 0.94 & 22.98 & 16.06 & 60.37 \\
    PCA & 0.87 & \textbf{23.14} & 14.95 & 83.59 \\
    \bottomrule
    \end{tabular}
    \end{subtable}
    \hfill
    \begin{subtable}[t]{0.48\textwidth}
    \centering
    \caption{Alignment Loss Weight}
    \begin{tabular}{@{}ccccc@{}}
    \toprule
    Weight $\lambda$ & rFID$\downarrow$ & PSNR$\uparrow$ & FID$\downarrow$ & IS$\uparrow$ \\
    \midrule
    0.001 & \textbf{0.89} & 22.78 & 17.57 & 75.20 \\
    0.005 & 1.02 & 22.98 & 16.93 & 78.01 \\
    0.01 & 1.31 & \textbf{23.12} & 13.67 & 82.13 \\
    0.05 & 1.81 & 21.82 & \textbf{12.30} & \textbf{92.48} \\
    \bottomrule
    \end{tabular}
    \end{subtable}
    \label{tab:ablation}
    \vspace{-1em}
\end{table}

\paragraph{Downsampling Operators} We ablate the downsampling operators in Tab.~\ref{tab:ablation} (a). We adopt the same settings as in Tab.~\ref{tab:imagenet} using the model with the textual embeddings (Qwen) as the target distribution. Despite all being linear downsampling operators, PCA and Avg. Pooling perform worse than Random Projection. We hypothesize that this is because unlike Random Projection which preserves the structure of the data, both PCA and Avg. Pooling are likely to destroy the structure. Avg. Pooling performs especially poorly since it merges close features that are independent from the location.

\paragraph{Alignment Loss Weight} 
We apply different strengths of regularization by altering the alignment loss weight $\lambda$ in Tab.~\ref{tab:ablation} (b). 
As expected, larger weight implies heavier regularization, worse reconstruction, and easier generation. 
However, heavier regularization limits the generation performance and may even cause the GAN loss to collapse.
A trade-off exists between reconstruction and generation when the capacity of the model is limited.

% \paragraph{Weights for the Alignment Loss}

\section{Conclusion}
\label{sec:conclusion}

This paper introduced a novel method for aligning learnable latent spaces with arbitrary target distributions by leveraging pre-trained flow-based generative models as expressive priors.
Our approach utilizes a computationally tractable alignment loss, adapted from the flow matching objective, to guide latent variables towards the target distribution.
We theoretically established that minimizing this alignment loss serves as a proxy for maximizing a variational lower bound on the log-likelihood of the latents under the flow-defined target.
The effectiveness of our method is validated through empirical results, including controlled toy settings and large-scale ImageNet experiments.
Ultimately, this work provides a flexible and powerful framework for incorporating rich distributional priors, paving the way for more structured and interpretable representation learning.
A \textit{limitation}, and also a promising future direction, is that the selection of the optimal prior remains a challenge. While semantic priors are effective for image generation, we posit that no single ``silver bullet'' prior exists for all tasks; rather, the optimal choice is likely task-specific and need be further explored.

\bibliographystyle{plain}
\bibliography{references}

%%%%%%%%%%%%%%%%%%%%%%%%%%%%%%%%%%%%%%%%%%%%%%%%%%%%%%%%%%%%
\newpage
\appendix
\section{Complete Proof}
\label{app:proof}

\subsection{Complete Proof for Proposition 1}
\label{app:proof_prop_1}

We restate Proposition 1 for clarity and self-containedness.
\begin{propone}{1}
\label{prop:propone}
Let $\bm{v}_\theta: \R^{d_1} \times [0,1] \to \R^{d_1}$ be a given velocity field, and $p_{\mathrm{init}}$ be a base distribution. For any $\bm{y} \in \R^{d_1}$, the log-likelihood $\log p_1^{\bm{v}_\theta}(\bm{y})$ of $\bm{y}$ under the distribution induced by flowing $p_{\mathrm{init}}$ with $\bm{v}_\theta$ from $t=0$ to $t=1$, is lower-bounded as:
\begin{equation}
    \label{eq:prop_elbo_statement_app}
\log p_1^{\bm{v}_\theta}(\bm{y}) \ge C(\bm{y}) -  \lambda \mathcal{L}_{\text{align}}(\bm{y}; \theta),
\end{equation}
where $\lambda > 0$ is a constant, $\mathcal{L}_{\text{align}}(\bm{y}; \theta)$ is defined as
\begin{equation}
    \label{eq:loss_y_app_ref}
    \mathcal{L}_{\text{align}}(\bm{y}; \theta) = \E_{s \sim \U[0, 1], \bm{x}_0 \sim p_{\mathrm{init}}(\bm{x}_0)} \left[ \| (\bm{y} - \bm{x}_0) - \bm{v}_\theta((1-s)\bm{x}_0 + s\bm{y}, s) \|^2 \right],
\end{equation}
and $C(\bm{y})$ is a term dependent on $\bm{y}$ and $\bm{v}_\theta$, given by
\begin{equation}
    \label{eq:C_y_app_ref}
    C(\bm{y}) = - \mathbb{E}_{s \sim \mathcal{U}[0,1], \bm{x}_0 \sim p_{\mathrm{init}}} \left[ \mathrm{Tr}(\nabla_{\bm{z}} \bm{v}_\theta((1-s)\bm{x}_0 + s\bm{y}, s)) \right].
\end{equation}
\end{propone}

To prove this, we make the following assumptions:
\begin{assumption}[Properties of the Velocity Field]
    \label{ass:v_props_app}
    The velocity field $\bm{v}_\theta: \R^{d_1} \times [0,1] \to \R^{d_1}$ is continuous in $t$ and Lipschitz continuous in its spatial argument with continuously differentiable components, ensuring that $\nabla_{\bm{z}} \bm{v}_\theta(\bm{z}, s)$ exists and is bounded on compact sets.
\end{assumption}
\begin{assumption}[Variational Path Choice]
    \label{ass:path_choice_app}
    The variational paths used to construct the ELBO are straight lines $\bm{z}_s(\bm{x}_0, \bm{y}) = (1-s)\bm{x}_0 + s\bm{y}$ for $s \in [0,1]$, originating from $\bm{x}_0 \sim p_{\mathrm{init}}$ and terminating at $\bm{y}$. The velocity of such a path is $\dot{\bm{z}}_s(\bm{x}_0, \bm{y}) = \bm{y} - \bm{x}_0$.
\end{assumption}
\begin{assumption}[Variational Distribution Choice]
    \label{ass:q_choice_app}
    The variational distribution over the initial states $\bm{x}_0$ conditioned on $\bm{y}$ is chosen as $q(\bm{x}_0 | \bm{y}) = p_{\mathrm{init}}(\bm{x}_0)$.
\end{assumption}
\begin{assumption}[Weighting Factor in ELBO]
    \label{ass:lambda_choice_app}
    The time-dependent weighting factor $\lambda_s$ in the general ELBO formulation (Eq.~\ref{eq:proof_elbo_form_general_app} below) is chosen as a positive constant $\lambda_s = \lambda > 0$ for all $s \in [0,1]$.
\end{assumption}

\begin{remark}[Optimality of $\bm{v}_\theta$]
\label{remark:optimal}
    Sec.~\ref{sec:proof} introduces Assumption~\ref{ass:vtheta_fixed}, which states that $\bm{v}_\theta$ is optimally pre-trained such that $\bm{v}_\theta((1-s)\bm{x}_0 + s\bm{x}_1, s) = \bm{x}_1 - \bm{x}_0$ for $\bm{x}_1 \sim p_{\mathrm{data}}$. This assumption is not required for the mathematical validity of the ELBO in Proposition 1 itself, which holds for any $\bm{v}_\theta$ satisfying Assumption~\ref{ass:v_props_app}. However, Assumption~\ref{ass:vtheta_fixed} is crucial for interpreting why minimizing $\mathcal{L}_{\text{align}}(\bm{y}; \theta)$ drives $\bm{y}$ towards $p_{\mathrm{data}}$, as it implies $\mathcal{L}_{\text{align}}(\bm{y}; \theta) \approx 0$ if $\bm{y} \sim p_{\mathrm{data}}$.
\end{remark}

\begin{proof}
The proof is based on a variational approach to lower-bound the log-likelihood in continuous-time generative models. This technique has been established in the literature for Neural ODEs and continuous normalizing flows \cite{chen2018neural,grathwohl2018ffjord,liu2022flow}. 

For a target point $\bm{y}$, we consider a family of paths $\bm{z}_s(\bm{x}_0, \bm{y})$ parameterized by initial states $\bm{x}_0$ drawn from a proposal distribution $q(\bm{x}_0|\bm{y})$, where each path starts at $\bm{x}_0$ and ends at $\bm{y}$ (i.e., $\bm{z}_0(\bm{x}_0, \bm{y}) = \bm{x}_0$ and $\bm{z}_1(\bm{x}_0, \bm{y}) = \bm{y}$). The variational lower bound is derived by considering the path integral formulation of the likelihood. For any such family of paths with velocities $\dot{\bm{z}}_s(\bm{x}_0, \bm{y})$, the bound takes the form:
\begin{align}
    \label{eq:proof_elbo_form_general_app}
    \log p_1^{\bm{v}_\theta}(\bm{y}) \ge \mathbb{E}_{\bm{x}_0 \sim q(\cdot|\bm{y})} \Bigg[ & \log p_{\mathrm{init}}(\bm{x}_0) - \log q(\bm{x}_0|\bm{y}) \nonumber \\
    & - \int_0^1 \lambda_s \| \dot{\bm{z}}_s(\bm{x}_0, \bm{y}) - \bm{v}_\theta(\bm{z}_s(\bm{x}_0, \bm{y}), s) \|^2 \mathrm{d}s \nonumber \\
    & - \int_0^1 \mathrm{Tr}(\nabla_{\bm{z}_s} \bm{v}_\theta(\bm{z}_s(\bm{x}_0, \bm{y}), s)) \mathrm{d}s \Bigg].
\end{align}

We now apply our specific assumptions:
\begin{itemize}
    \item By Assumption~\ref{ass:path_choice_app}, the paths are $\bm{z}_s(\bm{x}_0, \bm{y}) = (1-s)\bm{x}_0 + s\bm{y}$, and their velocities are $\dot{\bm{z}}_s(\bm{x}_0, \bm{y}) = \bm{y} - \bm{x}_0$.
    \item By Assumption~\ref{ass:q_choice_app}, $q(\bm{x}_0|\bm{y}) = p_{\mathrm{init}}(\bm{x}_0)$. This causes the term $\log p_{\mathrm{init}}(\bm{x}_0) - \log q(\bm{x}_0|\bm{y})$ to vanish.
    \item By Assumption~\ref{ass:lambda_choice_app}, we set $\lambda_s = \lambda$, a positive constant for all $s \in [0,1]$.
\end{itemize}
Substituting these into Eq.~\ref{eq:proof_elbo_form_general_app}:
\begin{align}
\log p_1^{\bm{v}_\theta}(\bm{y}) \ge \mathbb{E}_{\bm{x}_0 \sim p_{\mathrm{init}}} \Bigg[ & - \int_0^1 \lambda \| (\bm{y} - \bm{x}_0) - \bm{v}_\theta((1-s)\bm{x}_0 + s\bm{y}, s) \|^2 \mathrm{d}s \nonumber \\
& - \int_0^1 \mathrm{Tr}(\nabla_{\bm{z}} \bm{v}_\theta((1-s)\bm{x}_0 + s\bm{y}, s)) \mathrm{d}s \Bigg]. \label{eq:proof_elbo_substituted_app}
\end{align}
Using the equivalence $\int_0^1 f(s) \mathrm{d}s = \mathbb{E}_{s \sim \mathcal{U}[0,1]}[f(s)]$ for integrable functions $f$, we can rewrite each term. The first term becomes:
\begin{align*}
    \mathbb{E}_{\bm{x}_0 \sim p_{\mathrm{init}}} \left[ - \lambda \int_0^1 \| (\bm{y} - \bm{x}_0) - \bm{v}_\theta((1-s)\bm{x}_0 + s\bm{y}, s) \|^2 \mathrm{d}s \right] \\
    = - \lambda \mathcal{L}_{\text{align}}(\bm{y}; \theta),
\end{align*}
using the definition of $\mathcal{L}_{\text{align}}(\bm{y}; \theta)$ from Eq.~\ref{eq:loss_y_app_ref}.
The second term becomes:
\begin{align*}
    \mathbb{E}_{\bm{x}_0 \sim p_{\mathrm{init}}} \left[ - \int_0^1 \mathrm{Tr}(\nabla_{\bm{z}} \bm{v}_\theta((1-s)\bm{x}_0 + s\bm{y}, s)) \mathrm{d}s \right] = C(\bm{y}),
\end{align*}
using the definition of $C(\bm{y})$ from Eq.~\ref{eq:C_y_app_ref}.
Combining these, the ELBO becomes:
\begin{equation}
\label{eq:proof_elbo_final_app}
\log p_1^{\bm{v}_\theta}(\bm{y}) \ge C(\bm{y}) - \lambda \mathcal{L}_{\text{align}}(\bm{y}; \theta).
\end{equation}
This completes the proof of Proposition 1. Our paper uses $\lambda=1$ for simplicity, yielding the bound $C(\bm{y}) - \mathcal{L}_{\text{align}}(\bm{y}; \theta)$.
\end{proof}

\subsection{Rigorous Analysis of \texorpdfstring{$C(\bm{y})$}{C(y)}}
\label{app:discussions_Cy_revised}

We now provide a rigorous analysis of the term $C(\bm{y})$ in the ELBO and establish conditions under which minimizing $\mathcal{L}_{\text{align}}(\bm{y}; \theta)$ leads to favorable behavior of $C(\bm{y})$.

\subsubsection{Geometric Interpretation of \texorpdfstring{$C(\bm{y})$}{C(y)}}

The term $C(\bm{y})$ represents the negative expected divergence of the velocity field $\bm{v}_\theta$ along straight-line variational paths from $\bm{x}_0 \sim p_{\mathrm{init}}$ to $\bm{y}$:
\begin{equation}
    \label{eq:Cy_def_appendix_rigorous}
    C(\bm{y}) = - \mathbb{E}_{s \sim \mathcal{U}[0,1], \bm{x}_0 \sim p_{\mathrm{init}}} \left[ \mathrm{Tr}(\nabla_{\bm{z}} \bm{v}_\theta((1-s)\bm{x}_0 + s\bm{y}, s)) \right].
\end{equation}

To understand its role, recall that in the exact likelihood computation for a flow model, we have:
\begin{equation}
\label{eq:exact_likelihood_flow}
\log p_1^{\bm{v}_\theta}(\bm{y}) = \log p_{\mathrm{init}}(\bm{x}_0^*(\bm{y})) - \int_0^1 \mathrm{Tr}(\nabla_{\bm{x}} \bm{v}_\theta(\bm{x}_s^*(\bm{y}), s)) \mathrm{d}s,
\end{equation}
where $\bm{x}_s^*(\bm{y})$ is the unique ODE trajectory satisfying $\dot{\bm{x}}_s^* = \bm{v}_\theta(\bm{x}_s^*, s)$ with $\bm{x}_1^*(\bm{y}) = \bm{y}$. The divergence integral measures the logarithmic volume change induced by the flow.

Our variational bound approximates this exact computation by averaging over straight-line paths rather than the true ODE trajectory. The quality of this approximation depends on how well the straight paths approximate the true flow geometry.

\subsubsection{Relationship Between \texorpdfstring{$C(\bm{y})$}{C(y)} and Distributional Alignment}

We establish the key relationship between $C(\bm{y})$ and the alignment quality measured by $\mathcal{L}_{\text{align}}(\bm{y}; \theta)$.

\begin{lemma}[Consistency of Variational Paths]
\label{lemma:path_consistency}
Under Assumption~\ref{ass:vtheta_fixed} (optimal $\bm{v}_\theta$), for $\bm{y} \sim p_{\mathrm{data}}$, the straight-line variational paths $\bm{z}_s = (1-s)\bm{x}_0 + s\bm{y}$ satisfy:
\begin{equation}
\mathbb{E}_{\bm{x}_0 \sim p_{\mathrm{init}}} \left[ \| (\bm{y} - \bm{x}_0) - \bm{v}_\theta(\bm{z}_s, s) \|^2 \right] = 0 \quad \forall s \in [0,1].
\end{equation}
Consequently, $\mathcal{L}_{\text{align}}(\bm{y}; \theta) = 0$ when $\bm{y} \sim p_{\mathrm{data}}$.
\end{lemma}

\begin{proof}
By Assumption~\ref{ass:vtheta_fixed}, for $\bm{y} \sim p_{\mathrm{data}}$ and $\bm{x}_0 \sim p_{\mathrm{init}}$, we have:
$$\bm{v}_\theta((1-s)\bm{x}_0 + s\bm{y}, s) = \bm{y} - \bm{x}_0.$$
Therefore, $\| (\bm{y} - \bm{x}_0) - \bm{v}_\theta((1-s)\bm{x}_0 + s\bm{y}, s) \|^2 = 0$ for all $\bm{x}_0$ and $s$, which implies the result.
\end{proof}

\begin{theorem}[Monotonic Behavior of the ELBO]
\label{thm:elbo_monotonic}
Consider two points $\bm{y}_1, \bm{y}_2 \in \R^{d_1}$ such that $\mathcal{L}_{\text{align}}(\bm{y}_1; \theta) > \mathcal{L}_{\text{align}}(\bm{y}_2; \theta)$. If the velocity field $\bm{v}_\theta$ is $L$-Lipschitz in its spatial argument and satisfies Assumption~\ref{ass:vtheta_fixed}, then:
\begin{equation}
\label{eq:C_difference_bound}
|C(\bm{y}_1) - C(\bm{y}_2)| \le L \cdot d_1 \cdot \frac{1}{2}\|\bm{y}_1 - \bm{y}_2\|,
\end{equation}
where the factor $\frac{1}{2}$ comes from $\mathbb{E}_{s \sim \mathcal{U}[0,1]}[s] = \frac{1}{2}$.
Moreover, if $\mathcal{L}_{\text{align}}(\bm{y}_1; \theta) - \mathcal{L}_{\text{align}}(\bm{y}_2; \theta) > \frac{L \cdot d_1}{2} \cdot \|\bm{y}_1 - \bm{y}_2\|$, then:
\begin{equation}
\log p_1^{\bm{v}_\theta}(\bm{y}_2) - \log p_1^{\bm{v}_\theta}(\bm{y}_1) > 0.
\end{equation}
\end{theorem}

\begin{proof}
From the definition of $C(\bm{y})$ in Eq.~\ref{eq:C_y_app_ref}:
\begin{equation}
    \begin{aligned}
    C(\bm{y}_1) - C(\bm{y}_2) &= -\mathbb{E}_{s,\bm{x}_0} \left[ \mathrm{Tr}(\nabla_{\bm{z}} \bm{v}_\theta((1-s)\bm{x}_0 + s\bm{y}_1, s)) \right] \\
    &\quad + \mathbb{E}_{s,\bm{x}_0} \left[ \mathrm{Tr}(\nabla_{\bm{z}} \bm{v}_\theta((1-s)\bm{x}_0 + s\bm{y}_2, s)) \right].
    \end{aligned}
\end{equation}
By the Lipschitz continuity of $\bm{v}_\theta$, its Jacobian $\nabla_{\bm{z}} \bm{v}_\theta(\bm{z}, s)$ has bounded operator norm $\|\nabla_{\bm{z}} \bm{v}_\theta(\bm{z}, s)\|_{op} \le L$. Therefore:
\begin{equation}
    \begin{aligned}
    |\mathrm{Tr}(\nabla_{\bm{z}} \bm{v}_\theta(\bm{z}_1, s)) - \mathrm{Tr}(\nabla_{\bm{z}} \bm{v}_\theta(\bm{z}_2, s))| &\le d_1 \cdot \|\nabla_{\bm{z}} \bm{v}_\theta(\bm{z}_1, s) - \nabla_{\bm{z}} \bm{v}_\theta(\bm{z}_2, s)\|_{op} \\
    &\le d_1 \cdot L \cdot \|\bm{z}_1 - \bm{z}_2\|.
    \end{aligned}
\end{equation}
Setting $\bm{z}_1 = (1-s)\bm{x}_0 + s\bm{y}_1$ and $\bm{z}_2 = (1-s)\bm{x}_0 + s\bm{y}_2$, we get $\|\bm{z}_1 - \bm{z}_2\| = s\|\bm{y}_1 - \bm{y}_2\|$. Taking expectations yields Eq.~\ref{eq:C_difference_bound}.

For the second part, using the ELBO bound from Proposition~\ref{prop:propone}:
\begin{equation}
    \begin{aligned}
    \log p_1^{\bm{v}_\theta}(\bm{y}_2) - \log p_1^{\bm{v}_\theta}(\bm{y}_1) &\ge [C(\bm{y}_2) - \mathcal{L}_{\text{align}}(\bm{y}_2; \theta)] - [C(\bm{y}_1) - \mathcal{L}_{\text{align}}(\bm{y}_1; \theta)] \\
    &= [C(\bm{y}_2) - C(\bm{y}_1)] + [\mathcal{L}_{\text{align}}(\bm{y}_1; \theta) - \mathcal{L}_{\text{align}}(\bm{y}_2; \theta)].
    \end{aligned}
\end{equation}
Using the bound on $|C(\bm{y}_1) - C(\bm{y}_2)|$ and the condition on $\mathcal{L}_{\text{align}}(\bm{y}_1; \theta) - \mathcal{L}_{\text{align}}(\bm{y}_2; \theta)$, the result follows.
\end{proof}

\subsubsection{Analysis of the Idealized Case}

We address the mathematical singularity that arises in the idealized rectified flow case where $\bm{v}_\theta$ has the exact form $\bm{v}_\theta(\bm{z}, s) = \bm{z}/s$ for $s > 0$.

\begin{proposition}[Regularization by Neural Network Parameterization]
\label{prop:regularization_effect}
Let $\bm{v}_\theta$ be parameterized by a neural network with bounded weights. Then there exists a constant $M > 0$ such that:
\begin{equation}
\left|\mathrm{Tr}(\nabla_{\bm{z}} \bm{v}_\theta(\bm{z}, s))\right| \le M \quad \forall \bm{z} \in \text{compact sets}, \, s \in [\epsilon, 1]
\end{equation}
for any $\epsilon > 0$. Consequently, $C(\bm{y})$ is well-defined and finite.
\end{proposition}

\begin{proof}
Neural networks with bounded parameters have Lipschitz continuous components. The Jacobian $\nabla_{\bm{z}} \bm{v}_\theta(\bm{z}, s)$ inherits this boundedness on compact sets, preventing the $1/s$ singularity from occurring exactly. The trace is therefore bounded, ensuring $C(\bm{y})$ remains finite.
\end{proof}

\subsubsection{Practical Implications and Optimization Strategy}

Our analysis establishes that:

\textbf{1. Consistency Principle:} When $\bm{y} \sim p_{\mathrm{data}}$, both $\mathcal{L}_{\text{align}}(\bm{y}; \theta) = 0$ and $C(\bm{y})$ takes on the value appropriate for samples from the target distribution.

\textbf{2. Monotonicity Property:} Theorem~\ref{thm:elbo_monotonic} shows that sufficiently large reductions in $\mathcal{L}_{\text{align}}(\bm{y}; \theta)$ guarantee improvements in the ELBO lower bound, even accounting for changes in $C(\bm{y})$.

\textbf{3. Computational Tractability:} While computing $C(\bm{y})$ requires evaluating Jacobian traces, minimizing only $\mathcal{L}_{\text{align}}(\bm{y}; \theta)$ provides a computationally efficient proxy that, by Theorem~\ref{thm:elbo_monotonic}, leads to ELBO improvements under reasonable conditions.

\textbf{4. Robustness:} Proposition~\ref{prop:regularization_effect} ensures that practical neural network implementations avoid the theoretical singularities, making the method stable in practice.

This analysis demonstrates that minimizing $\mathcal{L}_{\text{align}}(\bm{y}; \theta)$ is not merely heuristic but has solid theoretical foundation as a strategy for maximizing the variational lower bound on $\log p_1^{\bm{v}_\theta}(\bm{y})$.

\subsection{The Significance of Assumptions}
\label{app:significance_assumptions}

The derivation of Proposition~\ref{prop:propone} and its interpretation rely on several assumptions, as listed in Sec.~\ref{app:proof_prop_1}. In this section, we discuss the significance of each assumption.

\textbf{Assumption~\ref{ass:v_props_app} (Properties of the Velocity Field):}
Lipschitz continuity of $\bm{v}_\theta$ in its spatial argument ensures that the ODE $\dot{\bm{z}}_t = \bm{v}_\theta(\bm{z}_t, t)$ has unique solutions, fundamental for defining $p_1^{\bm{v}_\theta}(\bm{y})$. Differentiability is required for the Jacobian $\nabla_{\bm{z}} \bm{v}_\theta$ to exist, and thus for the divergence term $\mathrm{Tr}(\nabla_{\bm{z}} \bm{v}_\theta)$ in the ELBO to be well-defined. These are standard regularity conditions for flow-based models. Without them, the ELBO formulation would be ill-defined.

\textbf{Assumption~\ref{ass:path_choice_app} (Variational Path Choice):}
The choice of straight-line paths $\bm{z}_s(\bm{x}_0, \bm{y}) = (1-s)\bm{x}_0 + s\bm{y}$ is a specific variational decision. This leads to the path velocity $\dot{\bm{z}}_s = \bm{y} - \bm{x}_0$, which is key to the definition of $\mathcal{L}_{\text{align}}(\bm{y}; \theta)$. This assumption is thus crucial for the specific form of $\mathcal{L}_{\text{align}}$ used.

\textbf{Assumption~\ref{ass:q_choice_app} (Variational Distribution Choice):}
Setting $q(\bm{x}_0 | \bm{y}) = p_{\mathrm{init}}(\bm{x}_0)$ greatly simplifies the ELBO by causing the term $\log p_{\mathrm{init}}(\bm{x}_0) - \log q(\bm{x}_0|\bm{y})$ to vanish. This common choice implies the ELBO considers paths originating from the prior, without inferring a specific $\bm{x}_0$ for each $\bm{y}$. While simplifying, this choice affects the ELBO's tightness.

\textbf{Assumption~\ref{ass:lambda_choice_app} (Weighting Factor in ELBO):}
Choosing $\lambda_s = \lambda$ makes the loss term in the ELBO directly correspond to $\mathcal{L}_{\text{align}}$. A time-dependent $\lambda_s > 0$ is also valid and could yield a tighter bound or differentially weight errors across time $s$. The constant $\lambda$ ensures a direct link to the standard L2 norm in $\mathcal{L}_{\text{align}}$. This choice affects the ELBO's value but not its validity as a lower bound.

\textbf{Assumption~\ref{ass:vtheta_fixed} (Optimality of $\bm{v}_\theta$):}
As detailed in Remark~\ref{remark:optimal} of Sec.~\ref{app:proof_prop_1}, this assumption is not necessary for the mathematical derivation of Proposition 1 itself; the ELBO inequality holds for any $\bm{v}_\theta$ satisfying Assumption~\ref{ass:v_props_app}. However, Assumption~\ref{ass:vtheta_fixed} is paramount for the \emph{interpretation} and \emph{effectiveness} of minimizing $\mathcal{L}_{\text{align}}(\bm{y}; \theta)$ as a strategy to align $\bm{y}$ with $p_{\mathrm{data}}$. If $\bm{v}_\theta$ is optimal as defined, then $\mathcal{L}_{\text{align}}(\bm{y}; \theta)$ will be minimized ideally to zero when $\bm{y}$ is drawn from $p_{\mathrm{data}}$. Consequently, minimizing this loss for $\bm{y}$ encourages $\bm{y}$ to conform to $p_{\mathrm{data}}$.

In essence, Assumptions~\ref{ass:v_props_app} through \ref{ass:lambda_choice_app} are primarily structural, defining the specific ELBO being analyzed. They ensure the bound is well-defined and takes the presented form. Assumption~\ref{ass:vtheta_fixed} concerning the optimality of $\bm{v}_\theta$ is interpretative, providing the rationale for why minimizing a component of this ELBO ($\mathcal{L}_{\text{align}}$) is a meaningful objective for achieving distributional alignment. The overall conclusion that minimizing $\mathcal{L}_{\text{align}}$ serves as a proxy for maximizing a log-likelihood lower bound relies on these assumptions.

\section{Additional Toy Examples}
\label{app:toy}

To further demonstrate the effectiveness of our proposed method, we present additional toy examples with diverse target distributions $p_{\text{data}}$: a Grid of Gaussians, Two Moons, Concentric Rings, a Spiral, and a Swiss Roll.
For each of these distributions, following the visualization style of Fig.~\ref{fig:toy_example}, we illustrate:
(a) The optimized variables $\bm{y}$ (red triangles) and samples from $p_{\text{data}}$ (blue dots), overlaid on the negative log-likelihood (NLL) landscape of $p_{\text{data}}$ (background heatmap showing $-\log p_{\text{data}}(\cdot)$).
(b) The landscape of the alignment loss $\mathcal{L}_{\text{align}}$ (background heatmap), with samples from $p_{\text{data}}$ (blue dots).
(c) The evolution of $\mathcal{L}_{\text{align}}(\bm{y}; \theta)$ (blue solid line) and the true NLL $-\log p_{\text{data}}(\bm{y})$ (red dashed line) during the optimization of $\bm{y}$.

For the Grid of Gaussians, which is also a mixture of Gaussians, the NLL $-\log p_{\text{data}}(\bm{y})$, is computed analytically. For the other distributions (Two Moons, Concentric Rings, Spiral, and Swiss Roll), where an analytical form for $p_{\text{data}}$ is not readily available, we estimate the NLL using Kernel Density Estimation (KDE). This estimation is based on $N=100,000$ samples drawn from the respective $p_{\text{data}}$ and employs a Gaussian kernel with a bandwidth of $h=0.1$.
The probability density $\hat{p}_{\text{data}}(\mathbf{x})$ at a point $\mathbf{x}$ is estimated as:
\begin{equation}
    \hat{p}_{\text{data}}(\mathbf{x}) = \frac{1}{N h^d} \sum_{i=1}^{N} K\left(\frac{\mathbf{x} - \mathbf{x}_i}{h}\right),
    \label{eq:kde_appendix}
\end{equation}
where $\mathbf{x}_i$ are the $N$ samples drawn from $p_{\text{data}}$, $d$ is the dimensionality (here, $d=2$), and $K(\cdot)$ is the Gaussian kernel function. The NLL for an optimized variable $\bm{y}$ is then approximated by $-\log\left(\hat{p}_{\text{data}}(\bm{y})\right)$. This provides an empirical measure of how well $\bm{y}$ aligns with the target distribution as estimated by KDE.

The results for these additional toy examples are comprehensively presented in Fig.~\ref{fig:more_toy_examples}. Each row in this figure corresponds to one of the five target distributions. The left column (a,d,g,j,m) shows that the optimized variables $\bm{y}$ (red triangles) successfully converge to the high-density (low-NLL) regions of $p_{\text{data}}$. The middle column (b,e,h,k,n) demonstrates that the landscape of our alignment loss $\mathcal{L}_{\text{align}}$ closely mirrors the NLL surface of $p_{\text{data}}$, with true data samples (blue dots) residing in low-loss areas. The right column (c,f,i,l,o) confirms the strong positive correlation between $\mathcal{L}_{\text{align}}$ and the NLL of $\bm{y}$, as both decrease concomitantly during optimization. Furthermore, Fig.~\ref{fig:toy_examples_training_progress} visualizes the optimization trajectory of $\bm{y}$ for the initial mixture of Gaussians (from Sec.~\ref{sec:toy}) alongside the five additional toy distributions. These sequential snapshots illustrate how minimizing $\mathcal{L}_{\text{align}}$ effectively steers the variables $\bm{y}$ from their initialization towards the intricate structures of the target distributions, reinforcing the robustness and efficacy of our alignment loss.

\begin{figure*}[p]
    \centering
    % Grid of Gaussians
    \begin{subfigure}{\textwidth}
        \centering
        \makebox[0pt][r]{\raisebox{-5em}{\rotatebox{90}{\parbox{10em}{\centering Grid of Gaussians}}}}
        \hspace{0mm}
        \begin{minipage}{0.3\textwidth}
            \centering
            {\small (a)}
            \includegraphics[width=\linewidth]{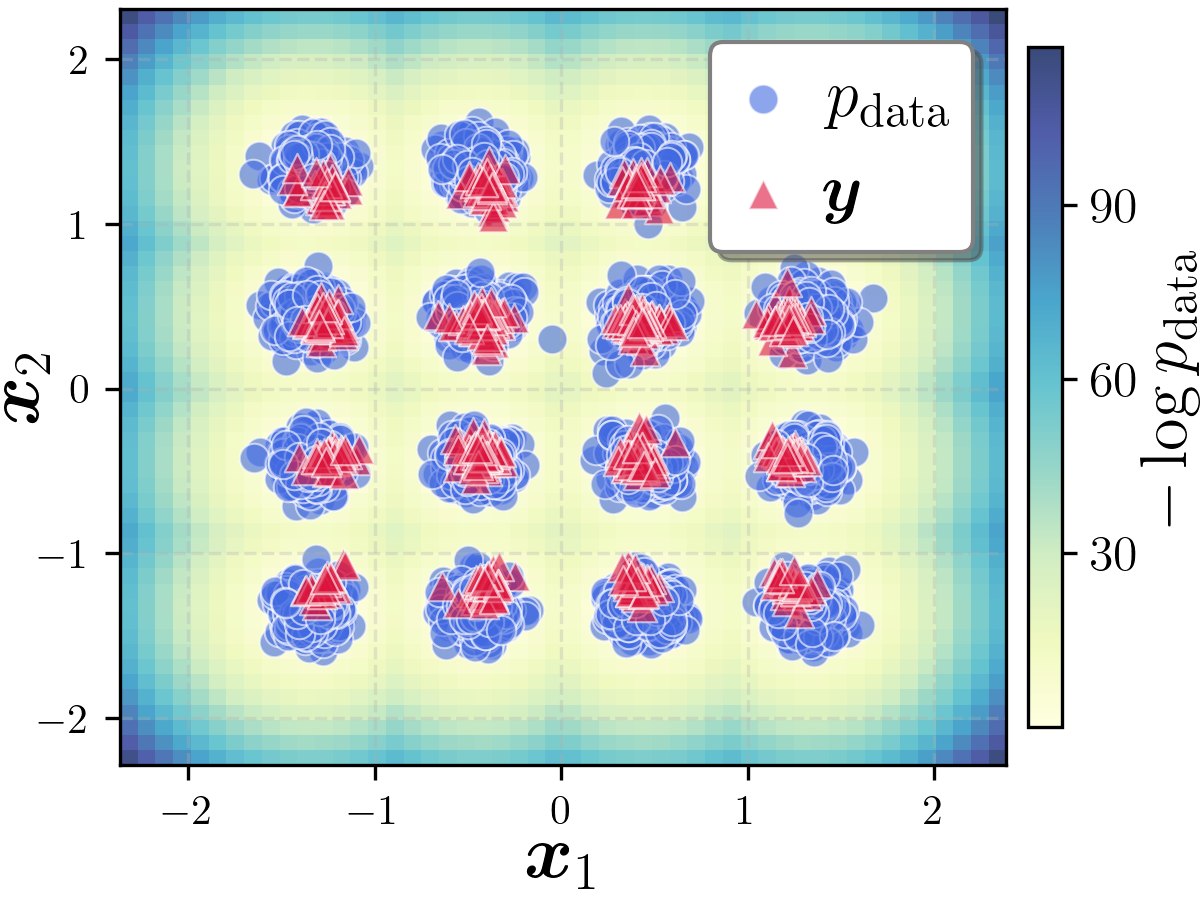}
        \end{minipage}
        \hfill
        \begin{minipage}{0.3\textwidth}
            \centering
            {\small (b)}
            \includegraphics[width=\linewidth]{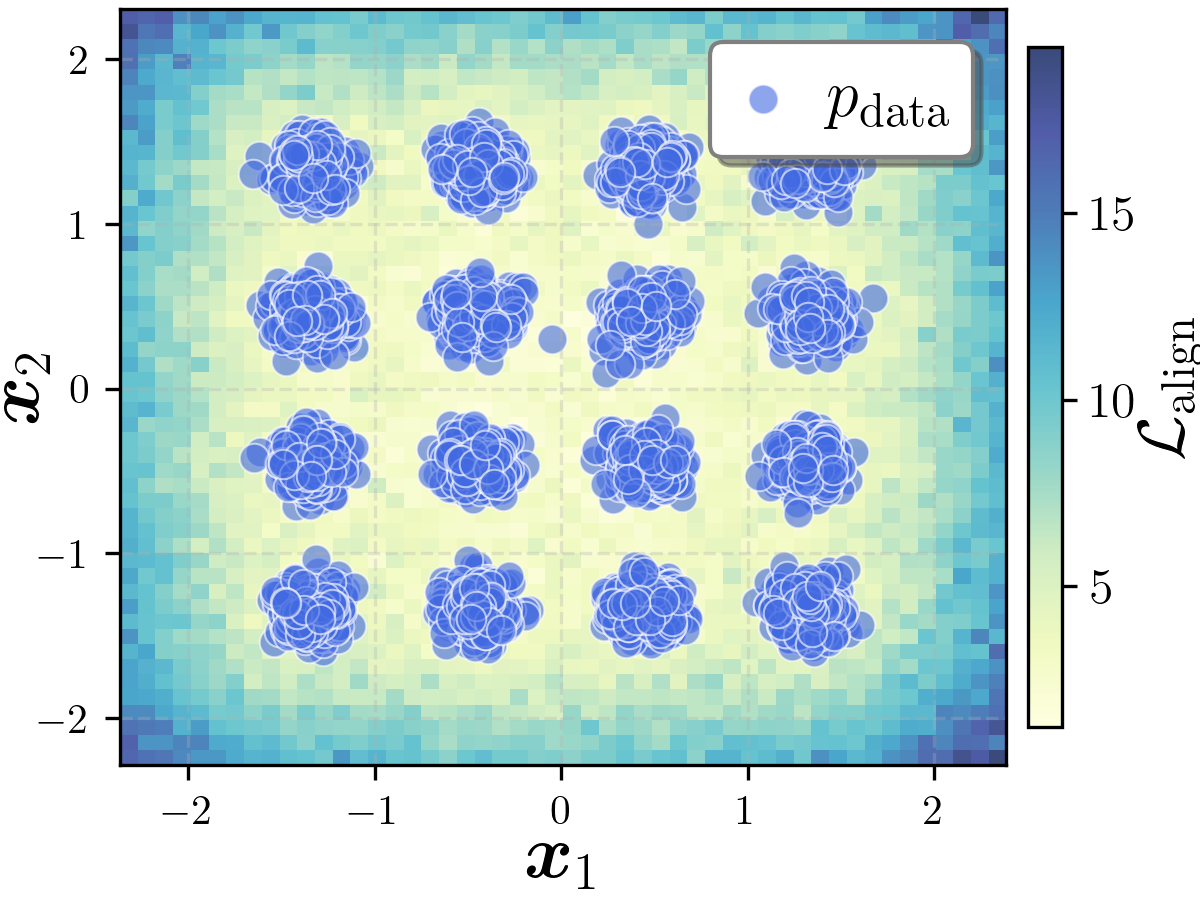}
        \end{minipage}
        \hfill
        \begin{minipage}{0.3\textwidth}
            \centering
            {\small (c)}
            \includegraphics[width=\linewidth]{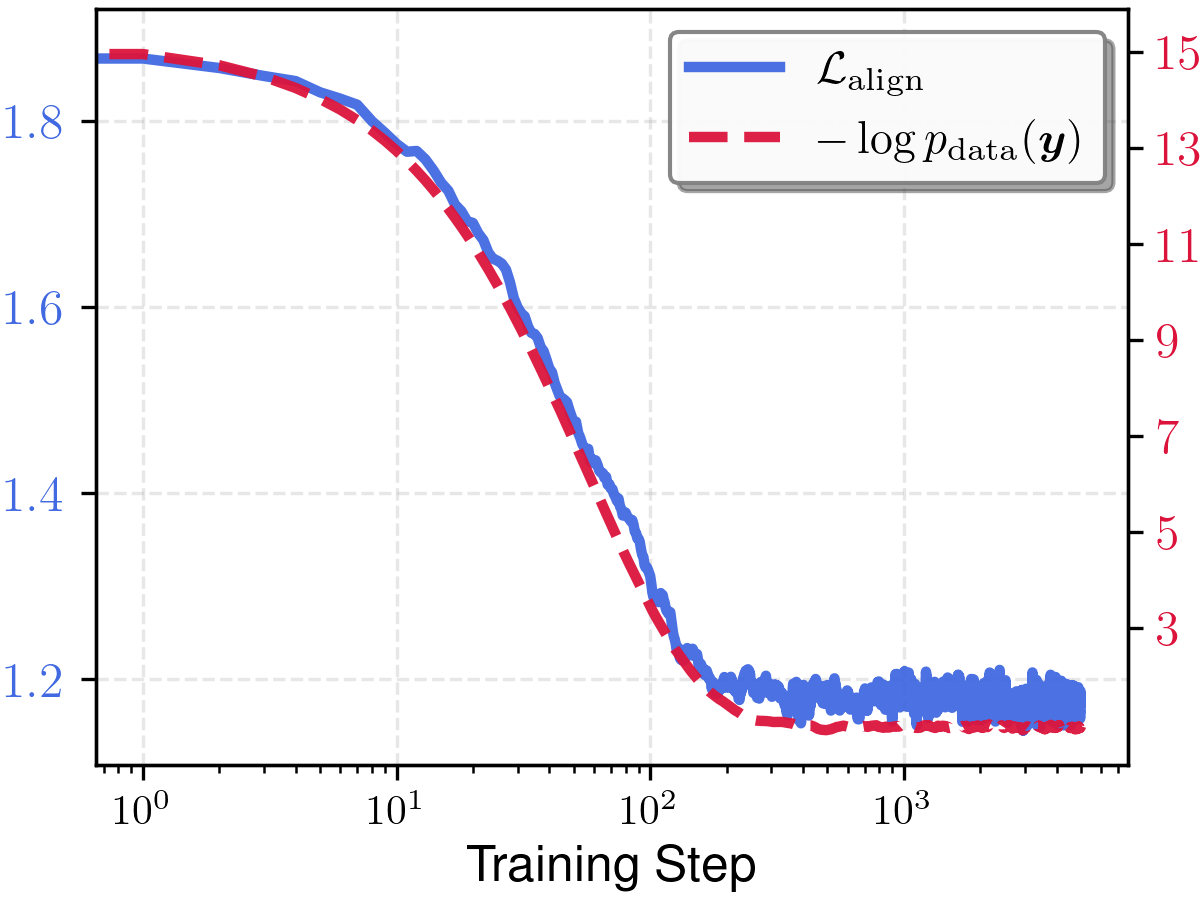}
        \end{minipage}
    \end{subfigure}
    
    % Two Moons
    \begin{subfigure}{\textwidth}
        \centering
        \makebox[0pt][r]{\raisebox{-5em}{\rotatebox{90}{\parbox{10em}{\centering Two Moons}}}}
        \hspace{0mm}
        \begin{minipage}{0.3\textwidth}
            \centering
            {\small (d)}
            \includegraphics[width=\linewidth]{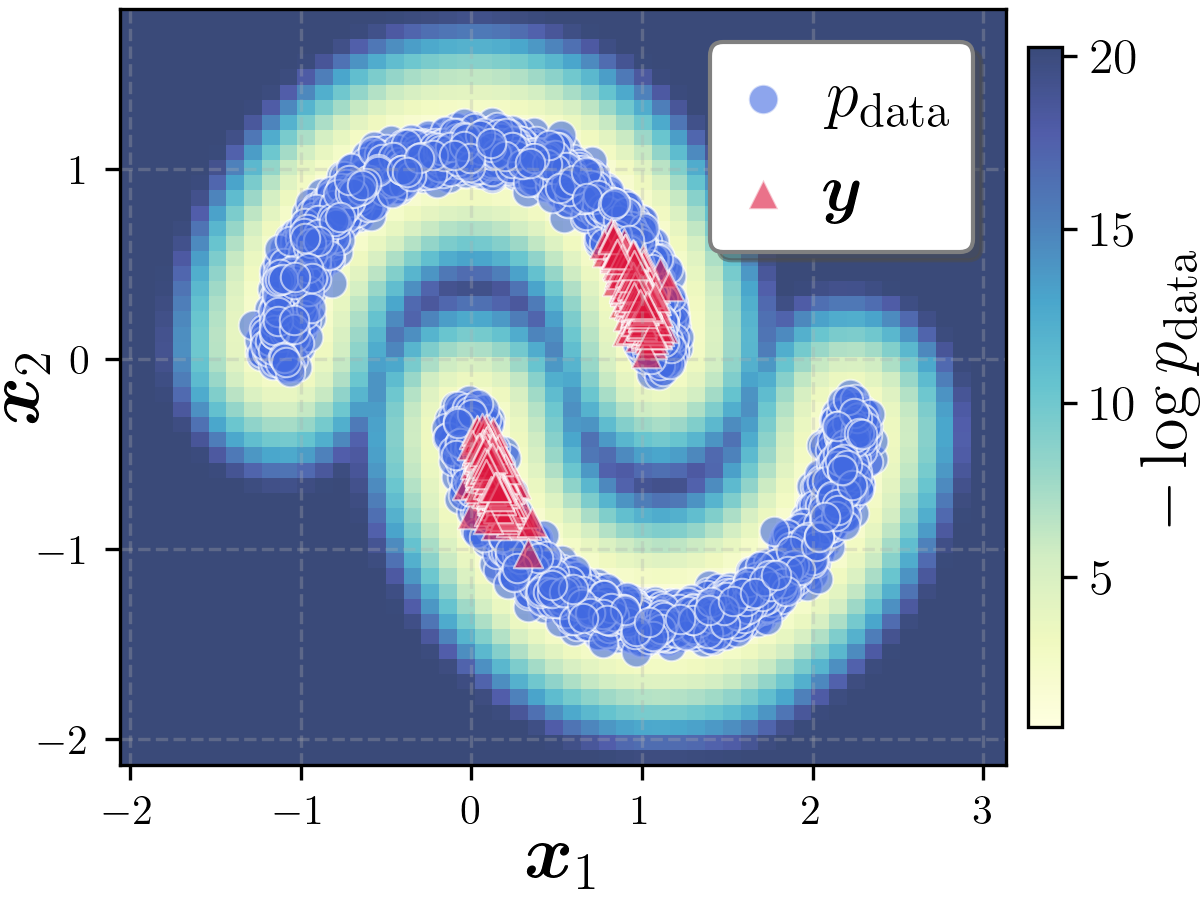}
        \end{minipage}
        \hfill
        \begin{minipage}{0.3\textwidth}
            \centering
            {\small (e)}
            \includegraphics[width=\linewidth]{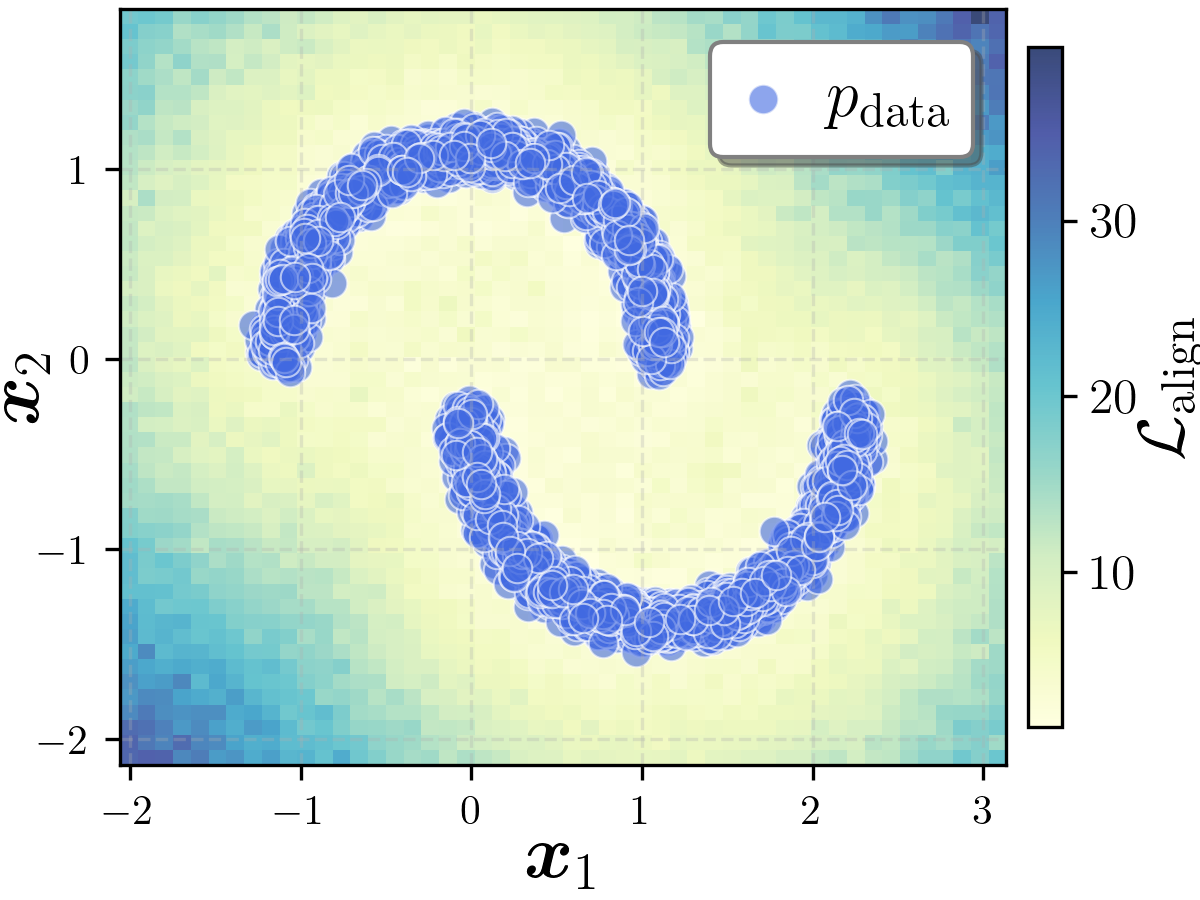}
        \end{minipage}
        \hfill
        \begin{minipage}{0.3\textwidth}
            \centering
            {\small (f)}
            \includegraphics[width=\linewidth]{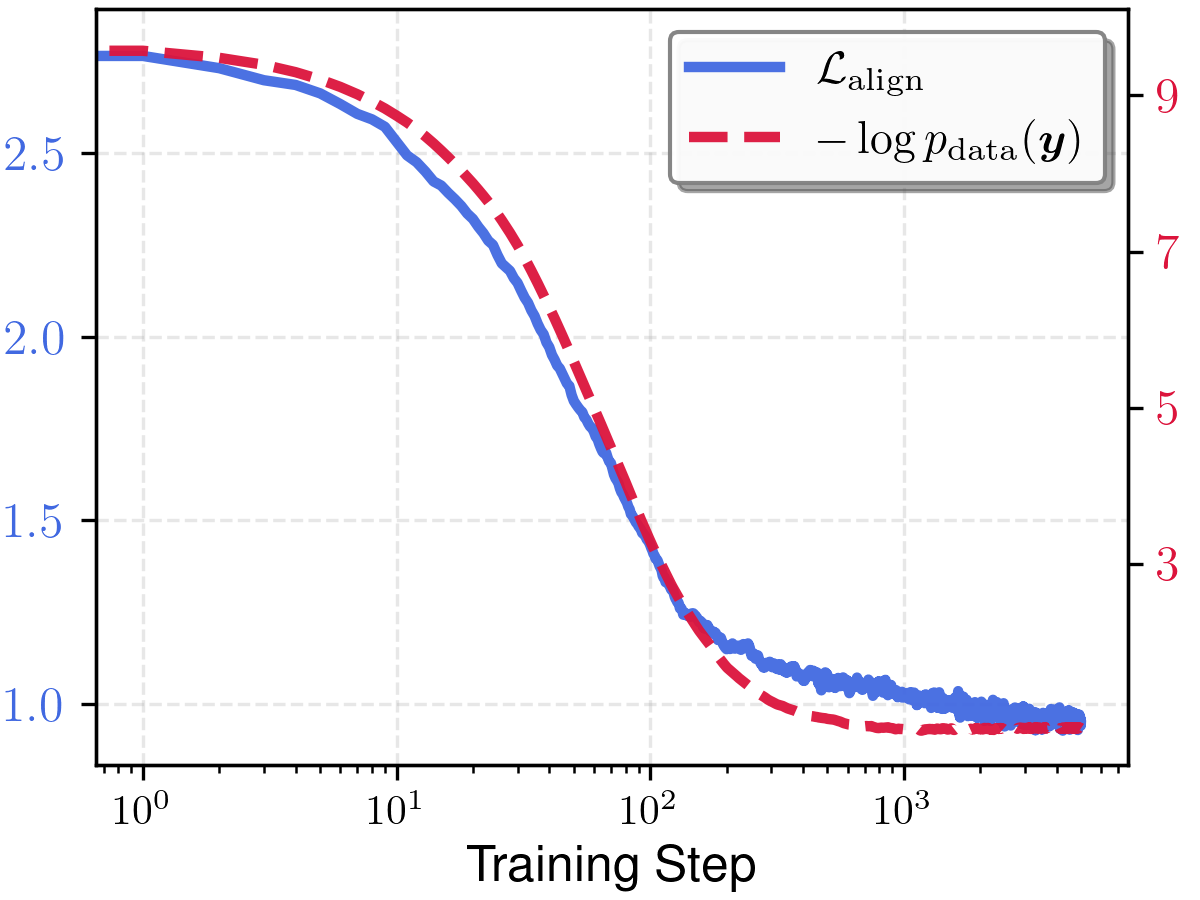}
        \end{minipage}
    \end{subfigure}
    
    % Concentric Rings
    \begin{subfigure}{\textwidth}
        \centering
        \makebox[0pt][r]{\raisebox{-5em}{\rotatebox{90}{\parbox{10em}{\centering Concentric Rings}}}}
        \hspace{0mm}
        \begin{minipage}{0.3\textwidth}
            \centering
            {\small (g)}
            \includegraphics[width=\linewidth]{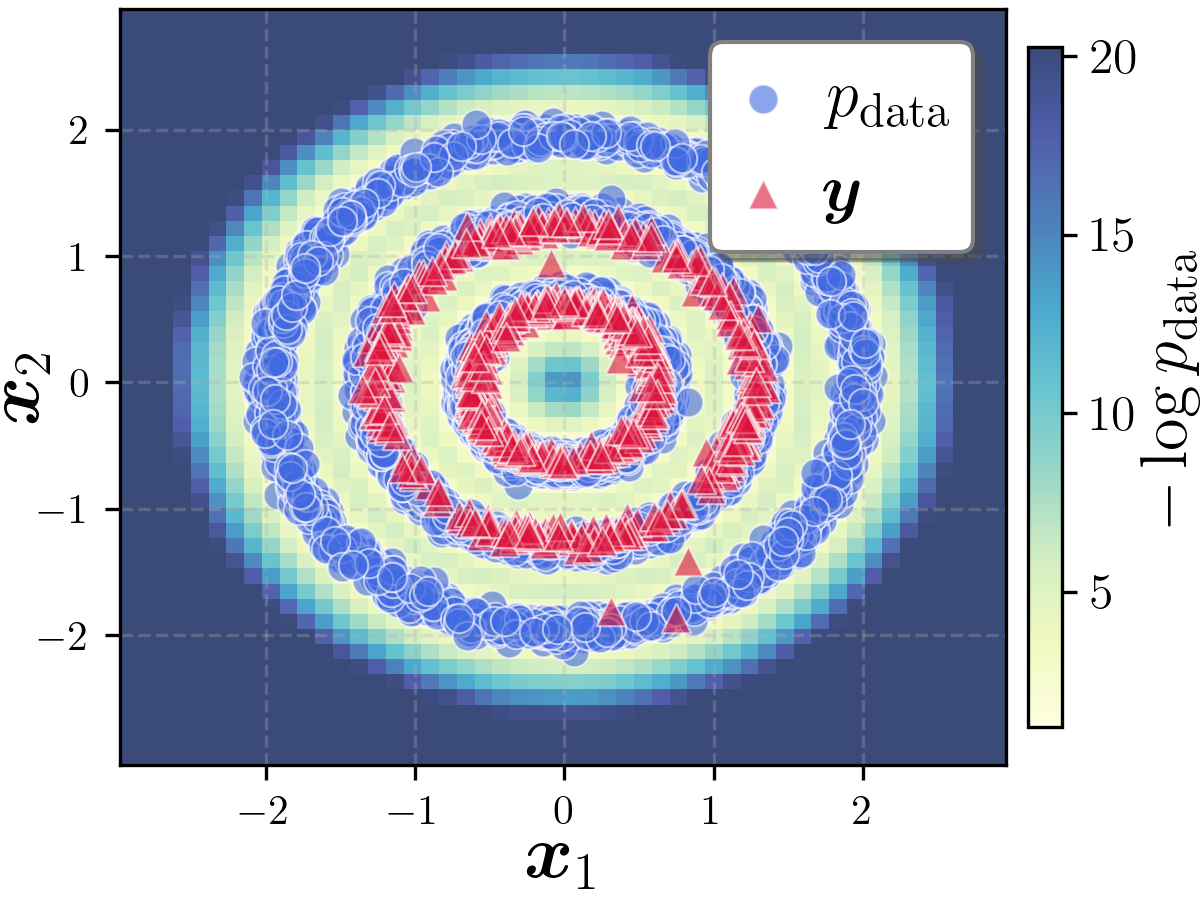}
        \end{minipage}
        \hfill
        \begin{minipage}{0.3\textwidth}
            \centering
            {\small (h)}
            \includegraphics[width=\linewidth]{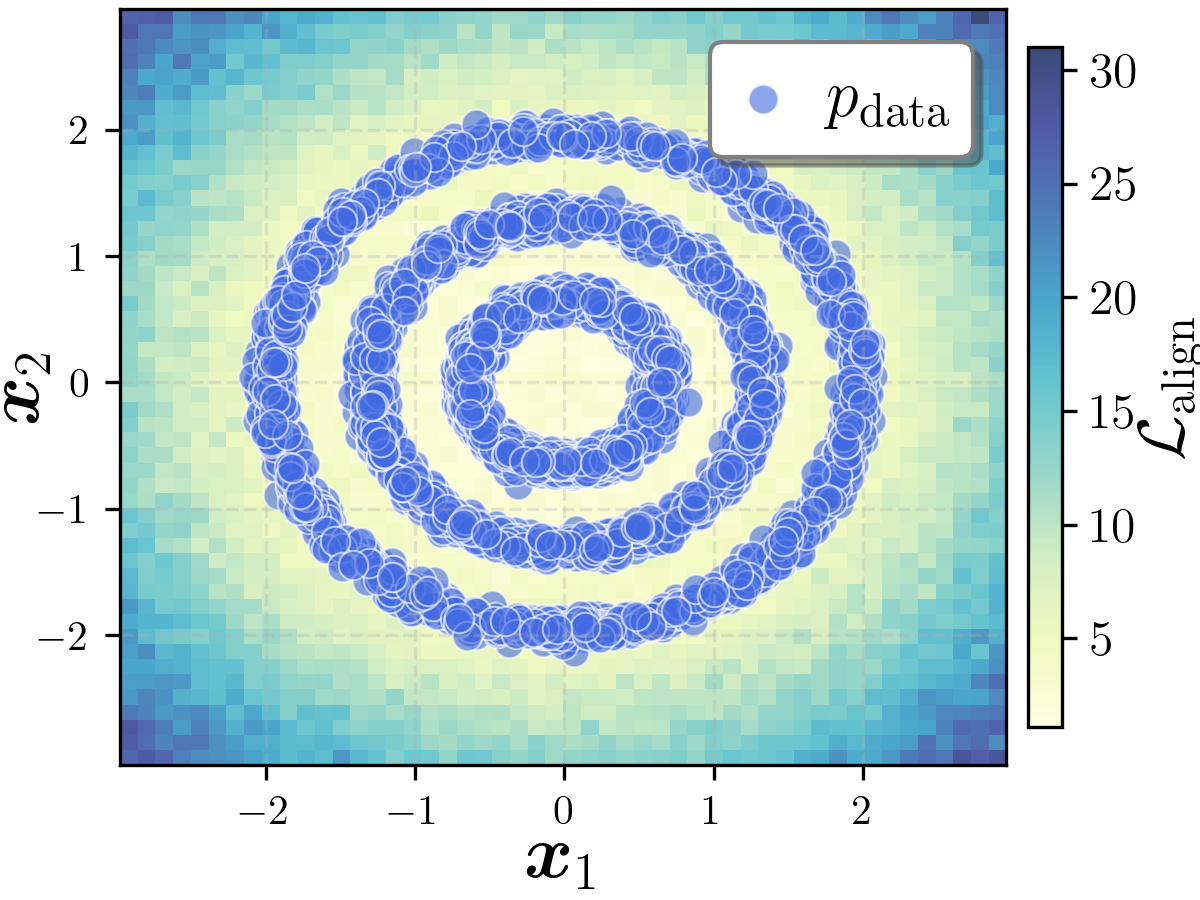}
        \end{minipage}
        \hfill
        \begin{minipage}{0.3\textwidth}
            \centering
            {\small (i)}
            \includegraphics[width=\linewidth]{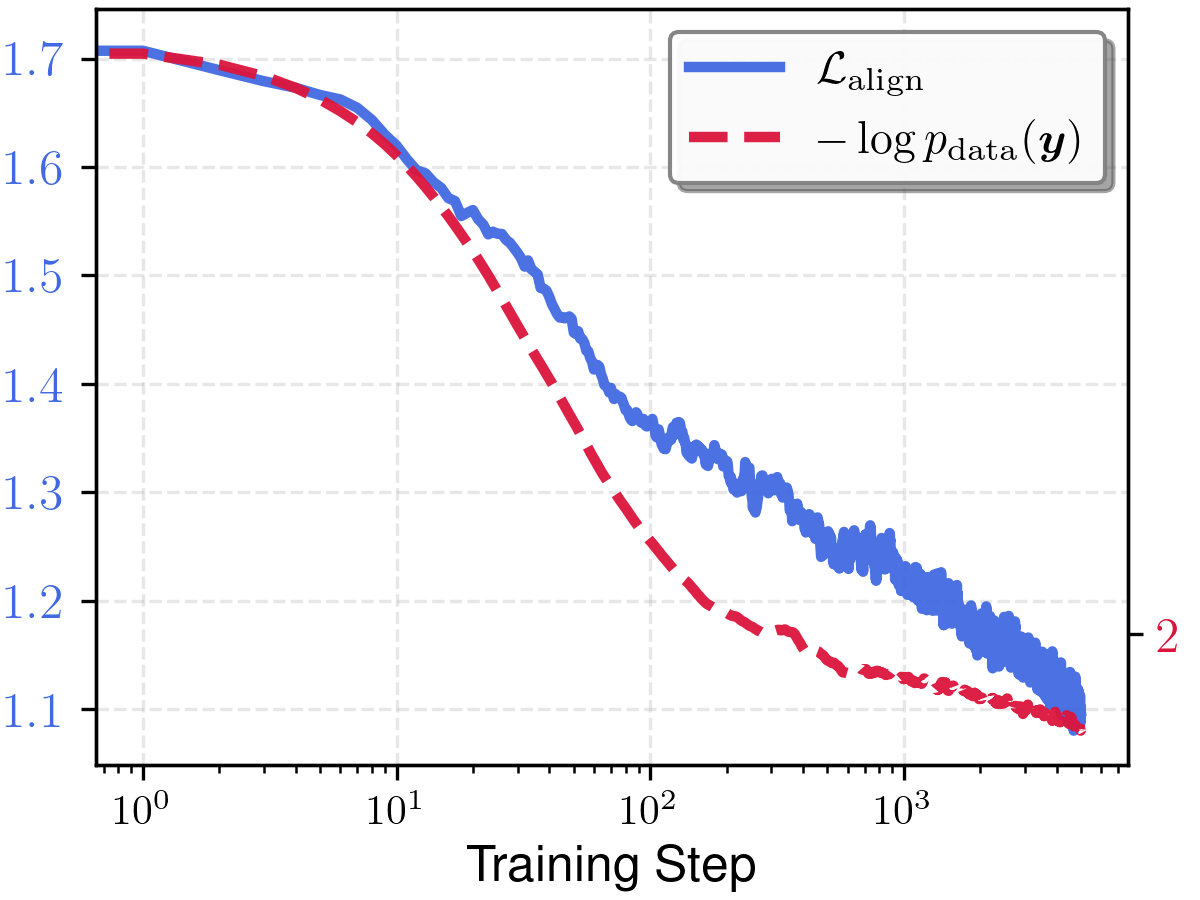}
        \end{minipage}
    \end{subfigure}
    
    % Spiral
    \begin{subfigure}{\textwidth}
        \centering
        \makebox[0pt][r]{\raisebox{-5em}{\rotatebox{90}{\parbox{10em}{\centering Spiral}}}}
        \hspace{0mm}
        \begin{minipage}{0.3\textwidth}
            \centering
            {\small (j)}
            \includegraphics[width=\linewidth]{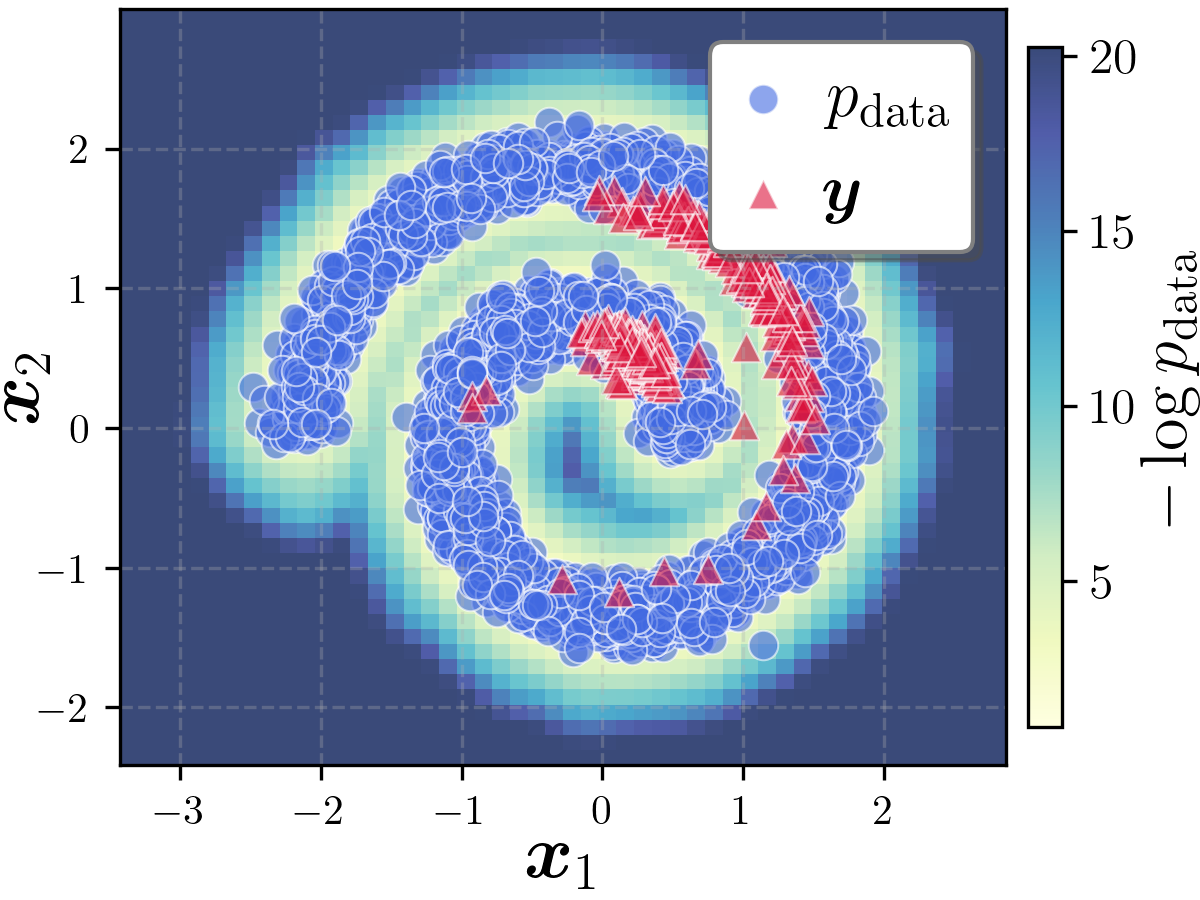}
        \end{minipage}
        \hfill
        \begin{minipage}{0.3\textwidth}
            \centering
            {\small (k)}
            \includegraphics[width=\linewidth]{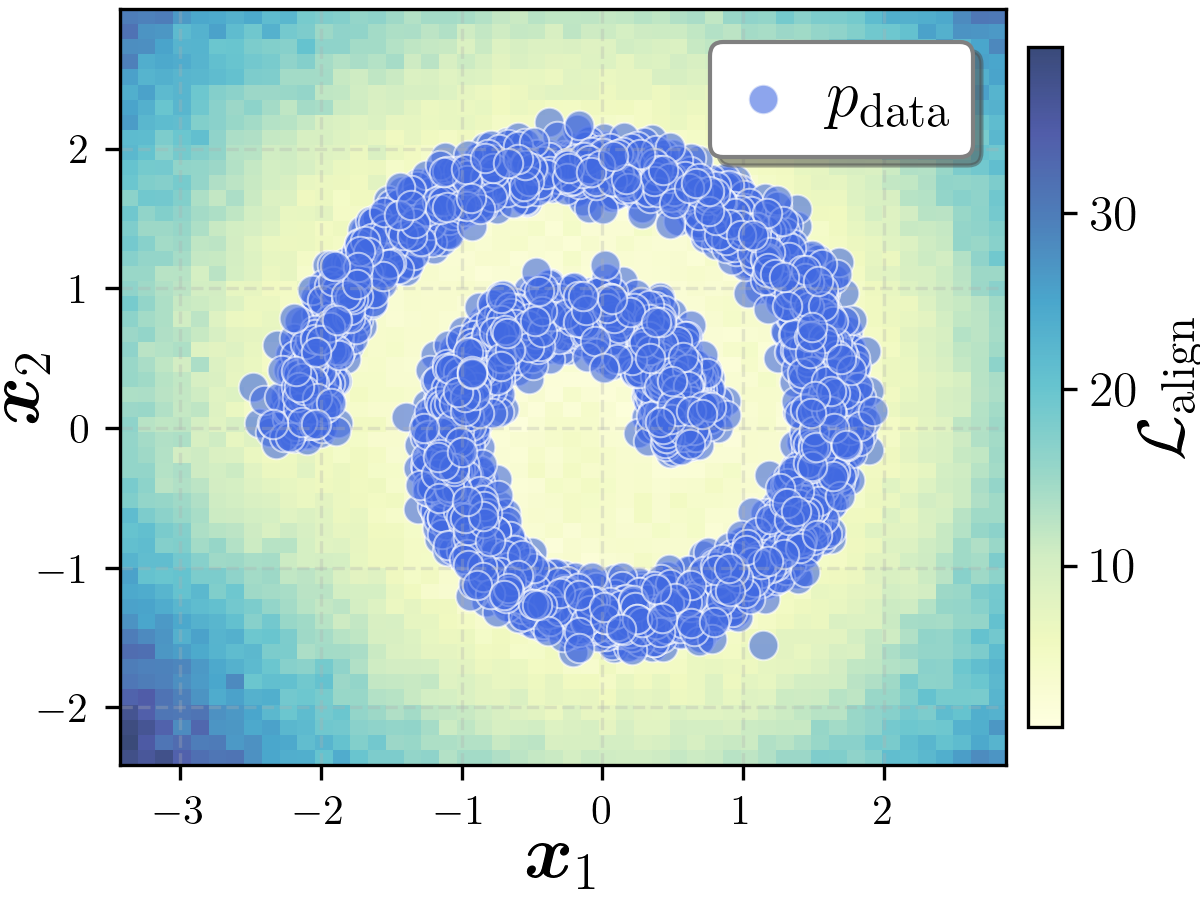}
        \end{minipage}
        \hfill
        \begin{minipage}{0.3\textwidth}
            \centering
            {\small (l)}
            \includegraphics[width=\linewidth]{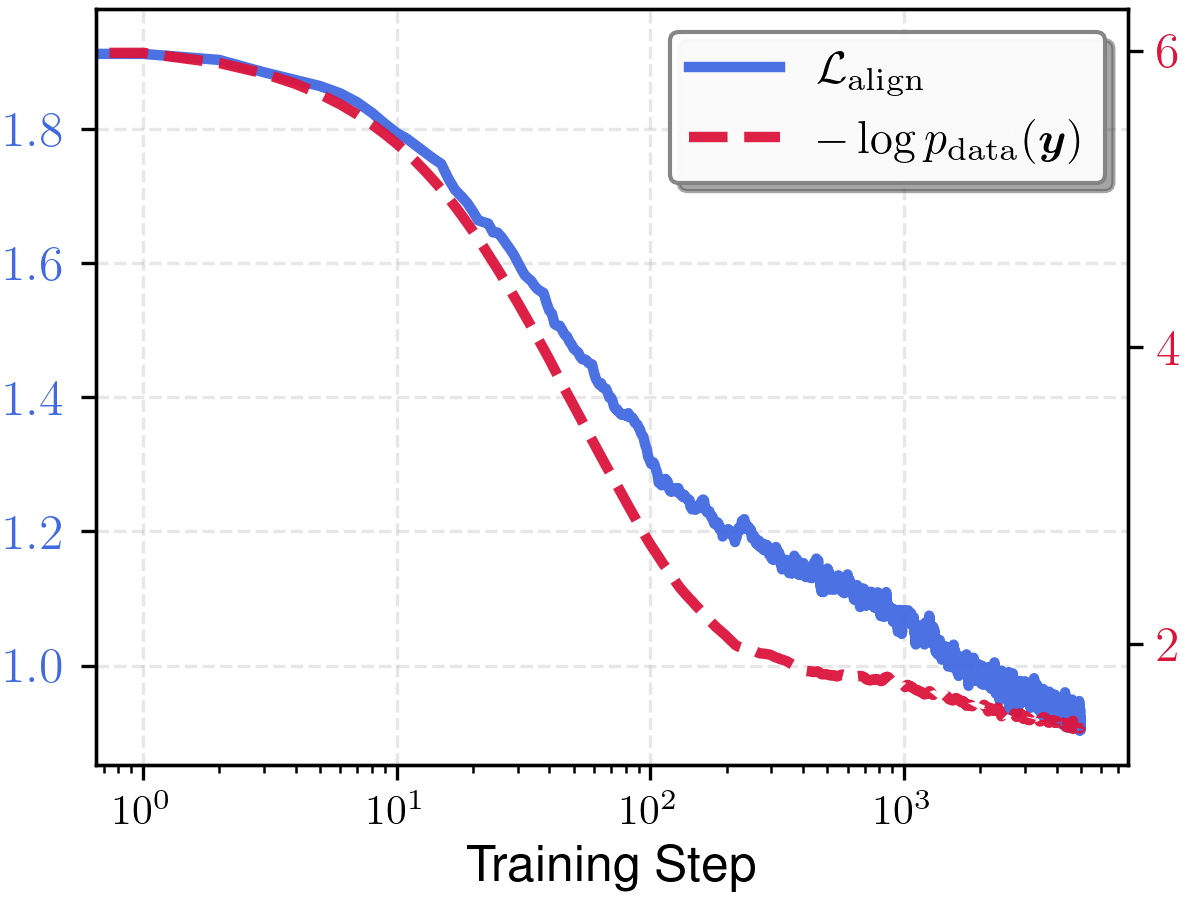}
        \end{minipage}
    \end{subfigure}
    
    % Swiss Roll
    \begin{subfigure}{\textwidth}
        \centering
        \makebox[0pt][r]{\raisebox{-5em}{\rotatebox{90}{\parbox{10em}{\centering Swiss Roll}}}}
        \hspace{0mm}
        \begin{minipage}{0.3\textwidth}
            \centering
            {\small (m)}
            \includegraphics[width=\linewidth]{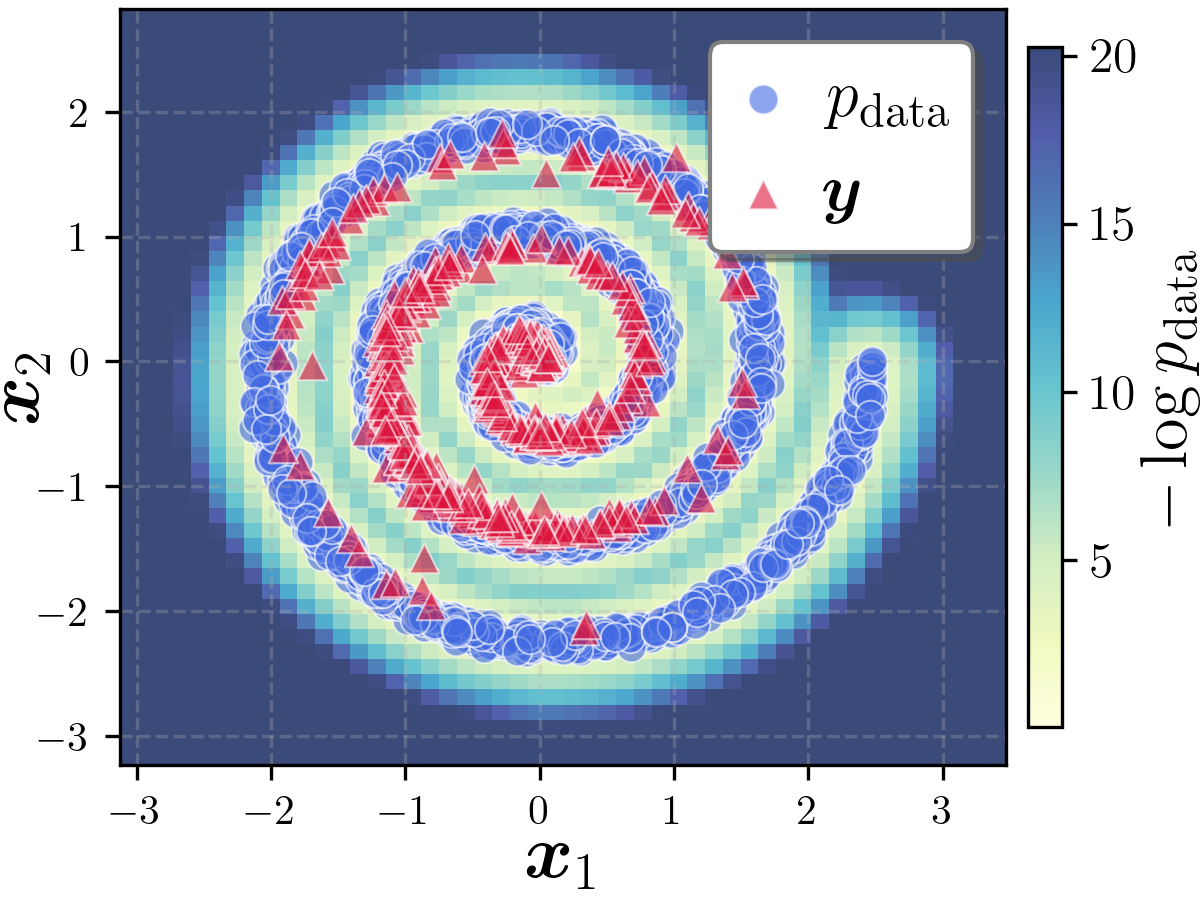}
        \end{minipage}
        \hfill
        \begin{minipage}{0.3\textwidth}
            \centering
            {\small (n)}
            \includegraphics[width=\linewidth]{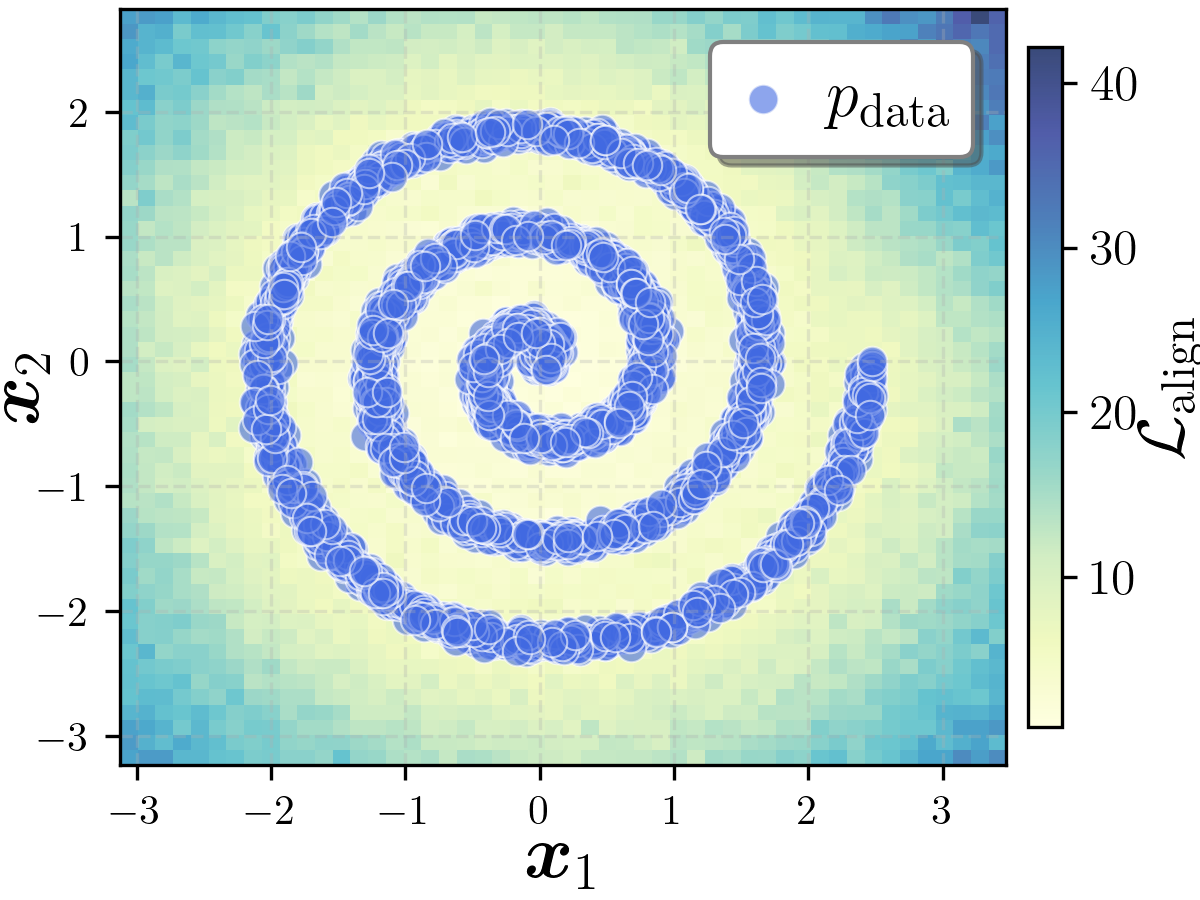}
        \end{minipage}
        \hfill
        \begin{minipage}{0.3\textwidth}
            \centering
            {\small (o)}
            \includegraphics[width=\linewidth]{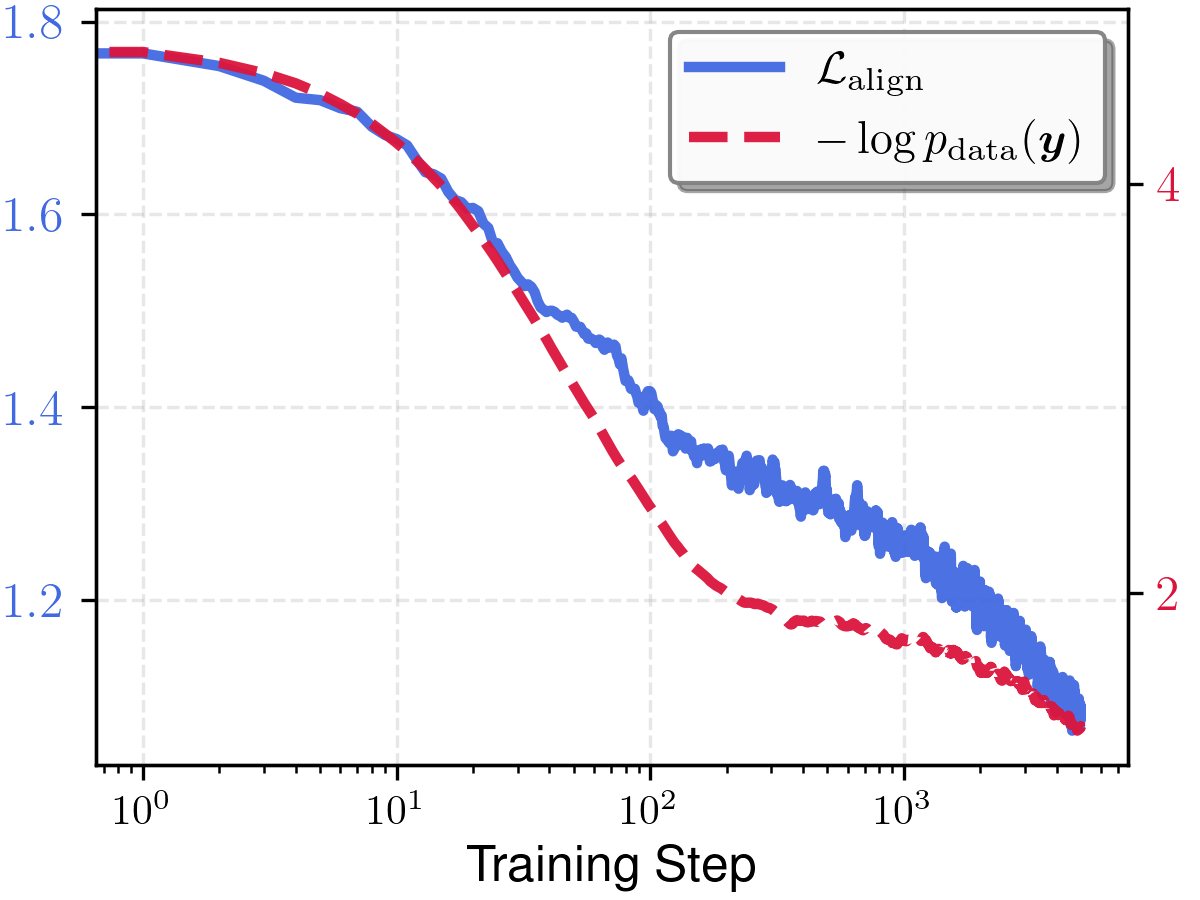}
        \end{minipage}
    \end{subfigure}
    
    \caption{Further illustrations of our method's performance on various 2D toy examples. Each row corresponds to a different target distribution $p_{\text{data}}$ (Grid of Gaussians, Two Moons, Concentric Rings, Spiral, and Swiss Roll).
    \textbf{Left column (a,d,g,j,m):} Optimized variables $\bm{y}$ (red triangles) and samples from $p_{\text{data}}$ (blue dots). The background heatmap visualizes the negative log-likelihood (NLL) $-\log p_{\text{data}}(\cdot)$, with $\bm{y}$ converging to low-NLL (high-density) regions.
    \textbf{Middle column (b,e,h,k,n):} The landscape of the alignment loss $\mathcal{L}_{\text{align}}$ (heatmap) with $p_{\text{data}}$ samples (blue dots). This landscape mirrors the NLL surface, and $p_{\text{data}}$ samples are concentrated in areas of low $\mathcal{L}_{\text{align}}$.
    \textbf{Right column (c,f,i,l,o):} Training curves for $\mathcal{L}_{\text{align}}(\bm{y}; \theta)$ (blue solid line) and NLL $-\log p_{\text{data}}(\bm{y})$ (red dashed line). Their strong positive correlation and concurrent decrease during optimization demonstrate that $\mathcal{L}_{\text{align}}$ effectively serves as a proxy for maximizing the log-likelihood of $\bm{y}$ under $p_{\text{data}}$.}
    \label{fig:more_toy_examples} % Added a label for the figure
\end{figure*}

\begin{figure}[htbp]
    \centering
    \begin{tikzpicture}[scale=0.9]
        \node[inner sep=0pt] (images) {%
            \begin{tabular}{@{}c@{\hspace{1pt}}c@{\hspace{1pt}}c@{\hspace{1pt}}c@{\hspace{1pt}}c@{\hspace{1pt}}c@{}}
                \includegraphics[width=0.154\textwidth]{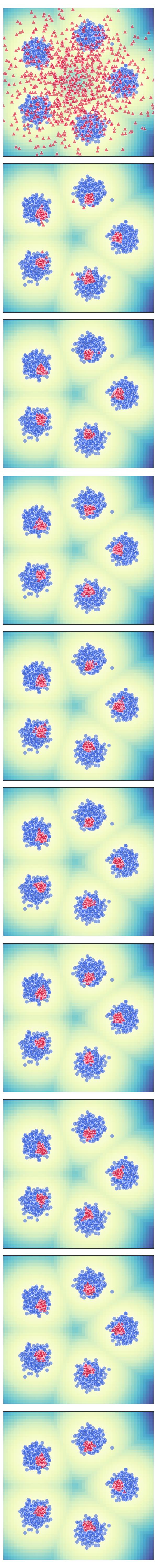} &
                \includegraphics[width=0.154\textwidth]{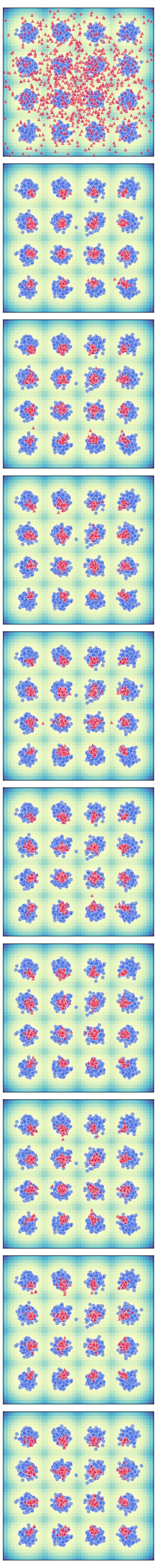} &
                \includegraphics[width=0.154\textwidth]{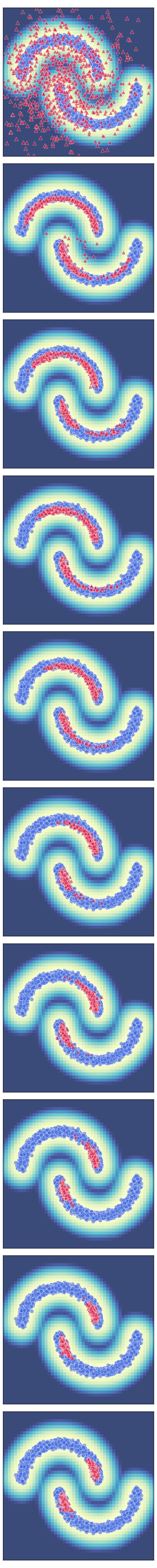} &
                \includegraphics[width=0.154\textwidth]{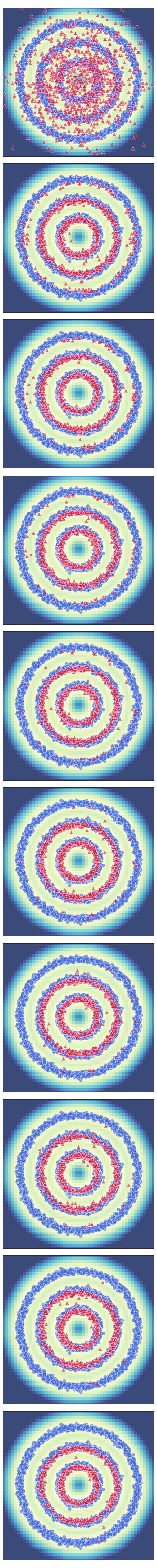} &
                \includegraphics[width=0.154\textwidth]{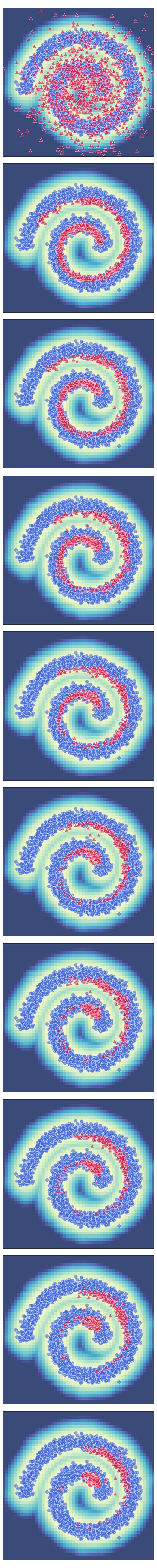} &
                \includegraphics[width=0.154\textwidth]{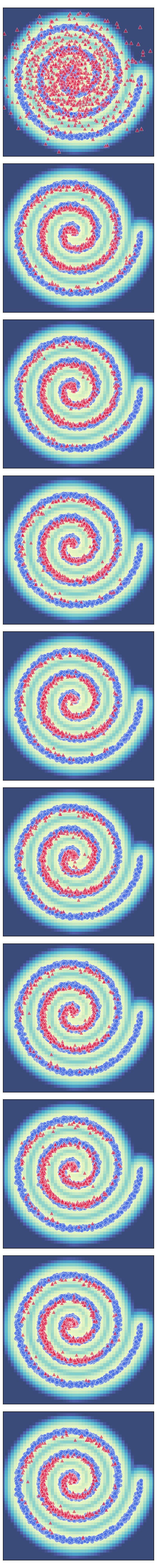}
            \end{tabular}%
        };
        \draw[-stealth, thick] ([xshift=-1em,yshift=-0.5em]images.north west) -- ([xshift=-1em,yshift=0.5em]images.south west) 
            node[midway, left=2mm, rotate=90] {Training Progress};
    \end{tikzpicture}
    \vspace{-0.2cm}
    \caption{Evolution of the optimized variables $\bm{y}$ (red triangles) during training across various toy examples. Each column represents a target distribution $p_{\text{data}}$. The training progress demonstrates how minimizing $\mathcal{L}_{\text{align}}$ guides $\bm{y}$ to converge towards low-NLL (high-density) regions of $p_{\text{data}}$.}
    \label{fig:toy_examples_training_progress} % Added a label for the figure
\end{figure}

\section{Implementation Details}
\label{app:implementation}

\subsection{Implementation Details of the Toy Example}

The primary toy example, illustrated in Figure~\ref{fig:toy_example}, utilizes a 2D Mixture of Gaussians (MoG) as the target data distribution \( p_{\text{data}}(\bm{x}) \). This MoG distribution consists of 5 components, each with an isotropic standard deviation of 0.3. The means of these Gaussian components are distributed evenly on a circle of radius 3.0. Prior to model training, samples drawn from this MoG distribution are normalized by dividing by their standard deviation, which is empirically computed from a large batch of 10 million samples. In addition to the MoG, our toy experiments also encompassed other 2D synthetic distributions, including Spiral, Moons, Concentric Rings, Swiss Roll, and Grid of Gaussians, to demonstrate the versatility of our approach. The general setup for the flow model and learnable latents applies across these various distributions.

The conditional flow model, denoted \( v_\phi(\bm{x}, t) \), is implemented using a MLP with AdaLN. This network has 2 input channels, 2 output channels, a hidden dimensionality of 512, and incorporates 4 residual blocks. The flow model is trained for 100,000 steps using the Adam optimizer (beta values of (0.9, 0.999) and no weight decay) with a constant learning rate of \(1 \times 10^{-4}\), and a batch size of 256.

A set of 1,000 learnable latent variables \(\{\bm{y}_i\}\) are initialized by sampling from a standard normal distribution \( \mathcal{N}(\bm{0}, \bm{I}) \). These latents are then optimized to align with the target distribution \( p_{\text{data}} \) by minimizing the alignment loss \( \mathcal{L}_{\text{align}} \). This alignment training phase also employs the Adam optimizer (betas=(0.9, 0.999), no weight decay), with a learning rate of \(1 \times 10^{-2}\), and runs for 5,000 steps.

\begin{table}[htbp]
    \centering
    \caption{Training Hyperparameters}
    \label{tab:hyperparams_comparison}
    \begin{tabular}{l>{\centering\arraybackslash}p{2cm}>{\centering\arraybackslash}p{2cm}>{\centering\arraybackslash}p{2cm}}
    \toprule
    \textbf{Hyperparameter} & \textbf{Flow} & \textbf{Autoencoder} & \textbf{MAR} \\
    \midrule
    Global Batch Size & \multicolumn{3}{c}{256} \\
    Steps & $1000k$ & $50k$ & $250k$ \\
    Optimizer & \multicolumn{3}{c}{AdamW} \\
    Base Learning Rate & \multicolumn{3}{c}{\num{1.0e-4}} \\
    LR Scheduler & Cosine & Cosine & Constant \\
    Warmup Steps & 2.5k & 2.5k & 62.5k \\
    Adam $\beta_1$ & \multicolumn{3}{c}{0.9} \\
    Adam $\beta_2$ & 0.95 & 0.95 & 0.999 \\
    Weight Decay & \num{1.0e-4} & \num{1.0e-4} & 0.02 \\
    Max Grad Norm & \multicolumn{3}{c}{1.0} \\
    Mixed Precision & \multicolumn{3}{c}{BF16} \\
    EMA Rate & \multicolumn{3}{c}{0.9999} \\
    \bottomrule
    \end{tabular}
\end{table}

\subsection{Implementation Details of the Flow Model}

The flow model \( \bm{v}_\theta(\bm{z}, t): \mathbb{R}^{d_1} \times [0,1] \rightarrow \mathbb{R}^{d_1} \) is implemented as a multi-layer perceptron (MLP) with 6 layers and 1024 hidden units per layer. The network employs GELU activation functions and incorporates time modulation through adaptive layer normalization (AdaLN) to handle the temporal dimension $t$. When dimension mismatch occurs between the latent space dimension $d_1$ and target feature space dimension $d_2$, fixed linear projection layers are applied to map target features to the appropriate dimension. These projection matrices are initialized with Gaussian weights scaled by $1/\sqrt{d_2}$ and remain frozen during training.

The flow model is trained using the flow matching objective on the target distribution $p_{\text{data}}$ for 1 million steps. During training, the model learns to predict velocity fields that transport samples from a standard Gaussian base distribution $\mathcal{N}(\bm{0}, \bm{I})$ to the target distribution along straight-line interpolation paths. The training employs mixed precision (BF16) with gradient clipping and exponential moving averages (EMA). Upon completion of training, the flow model parameters $\theta$ are frozen and used for subsequent latent space alignment. Detailed hyperparameters are provided in Table~\ref{tab:hyperparams_comparison}.

\subsection{Implementation Details of Autoencoders}

Our autoencoder architecture follows the SoftVQ design, which employs Vision Transformer (ViT) based encoder and decoder networks. The encoder utilizes a ViT-Large model with patch size 14 from DINOv2 \citep{oquab2023dinov2}, initialized with pre-trained weights and fine-tuned with full parameter updates during training. The decoder employs the same ViT-Large architecture but is initialized randomly without pre-trained weights.

The training process utilizes adversarial loss with a DINOv2-based discriminator, incorporating patch-based adversarial training with hinge loss formulation. Perceptual loss is applied using VGG features with a warmup period of $10k$ steps. The model is trained for $50k$ steps with cosine learning rate scheduling and exponential moving averages for stable training dynamics. Unlike SoftVQ, we do not employ the sample-level alignment loss (i.e., REPA loss), making our method more general and efficient. Detailed hyperparameters are provided in Table~\ref{tab:hyperparams_comparison}.

We followed the SoftVQ implementation as closely as possible. While we can reproduce almost identical reconstruction results, our tokenizer doesn't quite match the generation performance of the released pre-trained model, even after significant effort to optimize it. We believe this gap comes from differences in the cleaned-up code and the specific hardware we used for training. To keep things fair and validate the effectiveness of our method, we conduct all experiments on \textit{the same hardware with identical training settings}.

\subsection{Implementation Details of MAR}

We follow the original MAR-B implementation with several key modifications. We incorporate qk-norm in the attention mechanism and replace the diffusion head with a flow-based head trained using per-token flow matching loss. The original SD-KL-16 autoencoder is replaced with our trained autoencoders, applying input normalization with scaling factor 1.7052 estimated from sample batches.

Our model uses MAR-B architecture with $256 \times 256$ input images. The flow-based MLP head features adaptive layer normalization with 6 layers and 1024 hidden units per layer, identical to the original diffusion implementation. The model processes sequences of length 64 corresponding to our 64-token latent representation. More training details are provided in Table~\ref{tab:hyperparams_comparison}.

For inference, we employ an Euler sampler with 100 steps for the flow-based generation. The autoregressive sampling is limited to 64 steps. Generation uses batch size 256 and produces 50,000 images for evaluation. All evaluations use the standard toolkit from guided diffusion with FID and IS metrics computed at regular intervals during training.

\section{Additional Discussions for the Experiments}
\label{app:exp}

Here we provide additional discussions and analysis for the experiments presented in Section~\ref{sec:generation}.

\paragraph{Does Johnson-Lindenstrauss Lemma Really Hold?}
While the Johnson-Lindenstrauss (JL) lemma theoretically guarantees that random projections preserve distances with high probability, our experimental setup violates its conditions due to the large sample size relative to the target dimension. However, our results demonstrate that random projections can still preserve distributional structure to a sufficient extent for effective alignment. In our ablation study with Tab.~\ref{tab:ablation}a, random projection achieves the best performance with FID of 11.89 and IS of 102.23, significantly outperforming PCA (FID: 14.95, IS: 83.59) and average pooling (FID: 16.06, IS: 60.37). This suggests that the structure-preserving properties of random projections, even when the JL lemma doesn't strictly hold, are more beneficial than the variance-maximizing properties of PCA or the spatial averaging of pooling operations.
\paragraph{Continuous or Discrete?}
Our method demonstrates robustness across both continuous and discrete target distributions. Continuous semantic features from DinoV2 achieve the best generation performance among all variants in Tab.~\ref{tab:imagenet} and the discrete textual features from Qwen also achieve effective performance. In contrast, discrete VQ features perform poorly, likely due to structural limitations imposed by low dimensionality (8-dim). The collapse observed in discrete VQ experiments during training can be attributed to the insufficient capacity of the low-dimensional latent space to capture the complexity of ImageNet data while simultaneously satisfying the alignment constraint.
\paragraph{Why Textual Features Work?}
The surprising effectiveness of textual embeddings (Qwen) for visual generation warrants deeper analysis. Despite being trained on text data, Qwen embeddings achieve competitive generation performance (FID: 11.89 without CFG) and the best PSNR (23.12) among aligned methods. This suggests that high-quality textual representations capture abstract semantic structures that are transferable across modalities. The 896-dimensional Qwen embeddings provide a rich semantic space that can effectively constrain the visual latent space without being overly restrictive. This cross-modal transferability indicates that the structural benefits of alignment are not limited to within-modality features.
\paragraph{Is Generation Loss a Good Indicator?}
The training loss in generation of our aligned autoencoders is significantly lower than other models. However, we observe that lower training losses do not necessarily translate to better generation results, even for flow-based models where loss is proven to be a direct indicator for generation performance. This paradox can be attributed to the simplification of the latent space under strong alignment constraints. While simplified latent spaces are easier for generative models to sample from (hence lower training losses), they may sacrifice the diversity and fine-grained details necessary for high-quality generation. This suggests that generation quality depends not only on the ease of modeling the latent distribution but also on the expressiveness and diversity preserved in the aligned space.
\paragraph{How to Select the Prior?}
The optimal choice of target distribution remains an open research question. Our experiments suggest several guidelines: (1) Higher dimensionality generally enables better performance, as evidenced by the poor performance of 8-dimensional VQ features compared to higher-dimensional alternatives. (2) Semantic richness matters, but not necessarily complexity—simple textual features can match sophisticated visual features. (3) The structural properties of the target distribution (e.g., smoothness, cluster separation) may be more important than its semantic content for generation quality.

\end{document}